\documentclass[11pt]{article}

\usepackage[utf8]{inputenc} 
\usepackage[T1]{fontenc}    
\usepackage[colorlinks, linkcolor=blue, citecolor=black]{hyperref}       
\usepackage{url}            
\usepackage{booktabs}       
\usepackage{amsfonts}       
\usepackage{nicefrac}       
\usepackage{microtype}      
\usepackage{xcolor}         
\usepackage{mathtools, nccmath}
\usepackage{minitoc}
\usepackage[round]{natbib}
\usepackage[inline, nomargin, draft]{fixme}
\fxsetup{theme=color}
\definecolor{fxnote}{rgb}{0.858, 0.188, 0.478}

\usepackage{microtype}
\usepackage{graphicx}
\usepackage{booktabs} 

\usepackage{comment}
\usepackage{pgfplots}
\usepackage{url}

\usepackage{amsfonts}       
\usepackage{amsmath}
\usepackage{amsfonts}
\usepackage{amssymb}
\usepackage{comment}
\usepackage{pythonhighlight}

\usepackage{amsthm}
\usepackage{amscd}
\usepackage{t1enc}
\usepackage{enumerate}
\usepackage{indentfirst}
\usepackage{listings}
\usepackage{graphicx}
\usepackage{graphics}
\usepackage{pict2e}
\usepackage{epic}
\usepackage{epstopdf} 
\usepackage{listings}
\usepackage{minitoc}

\newcommand{\cA}{\mathcal{A}}

\newcommand{\cD}{\mathcal{D}}

\newcommand{\cG}{\mathcal{G}}
\newcommand{\cH}{\mathcal{H}}
\newcommand{\cI}{\mathcal{I}}
\newcommand{\cJ}{\mathcal{J}}

\newcommand{\cL}{\mathcal{L}}
\newcommand{\cM}{\mathcal{M}}

\newcommand{\cR}{\mathcal{R}}

\newcommand{\cY}{\mathcal{Y}}

\newcommand{\sysL}{L_{\mathrm{def}}^{0{-}1}}
\newcommand{\sysLhat}{\hat{L}_{\mathrm{def}}^{0{-}1}}
\newcommand{\bE}{\mathbb{E}}
\newcommand{\bP}{\mathbb{P}}
\newcommand{\bI}{\mathbb{I}}

\newcommand{\err}{\textrm{err}}

\renewcommand{\P}{\mathbf{P}}

\usepackage{amsmath}

\DeclareMathOperator*{\arginf}{arg\,inf}

\DeclareMathOperator*{\altargmin}{arg\,min}

\usepackage{empheq}

\renewcommand{\tilde}{\widetilde}
\renewcommand{\nu}{\vartheta}

\usepackage{float}
\usepackage{bm}
\usepackage{subcaption}
\usepackage{wrapfig}
\newtheorem{theorem}{Theorem}
\newtheorem{corollary }{Corollary }
\newtheorem*{theorem*}{Theorem}
\newtheorem{lemma}{Lemma}
\newtheorem{proposition}{Proposition}
\newtheorem{definition}{Definition}

\newtheorem{corollary}{Corollary}
\usepackage{apptools}
\AtAppendix{\counterwithin{lemma}{section}}
\AtAppendix{\counterwithin{theorem}{section}}
\AtAppendix{\counterwithin{corollary}{section}}
\theoremstyle{definition}
\usepackage{algorithm}
\usepackage{algorithmic}

\newtheorem{assumption}{Assumption}

\newenvironment{sproof}{%
  \proof}{\endproof}
\usepackage{environ}
\usepackage{tikz}
\usepackage{float}

\newcounter{todos}

\DeclareMathOperator*{\minimize}{\text{minimize}}

\renewcommand{\P}{\mathbb{P}}
\newcommand{\E}{\mathbb{E}}

\newcommand{\realizablesurrogate}{\texttt{RealizableSurrogate} }
\newcommand{\realizablesurrogatenosp}{\texttt{RealizableSurrogate}}

\usepackage{epstopdf} 
\usepackage{wrapfig}
\usepackage{algorithm}
\usepackage{algorithmic}
\usepackage{amsthm}
\usepackage{tabularx}
\usepackage{float}
\usepackage{setspace}
\usepackage{natbib}
\usepackage{fancyhdr}

\allowdisplaybreaks
\usepackage{authblk}

\usepackage[left=1in, right =1in, top=1in, bottom=1in]{geometry}
\pgfplotsset{compat=1.14}
\usepackage{pythonhighlight}
\usepackage[most]{tcolorbox}

\newtcbox{\mymath}[1][]{%
    nobeforeafter, math upper, tcbox raise base,
    enhanced, colframe=black!30!black,
    colback=white!30, boxrule=1pt,
    #1}

\title{\textbf{Who Should Predict? Exact Algorithms For Learning to Defer to Humans}}
\renewcommand{\headrulewidth}{0pt}
\author[1,2]{Hussein Mozannar}
\author[1,2]{Hunter Lang}
\author[1,3]{Dennis  Wei}
\author[1,3]{Prasanna  Sattigeri}
\author[1,3]{Subhro Das}
\author[1,2]{David Sontag}

\affil[1]{MIT-IBM Watson AI Lab}
\affil[2]{Massachusetts Institute of Technology}
\affil[3]{IBM Research}

\date{}
\begin{document}

%

%

\maketitle

\begin{abstract}
Automated AI classifiers should be able to defer the prediction to a human decision maker to ensure more accurate predictions. In this work, we jointly train a classifier with a \emph{rejector}, which decides on each data point whether the classifier or the human should predict. We show that prior approaches can fail to find a human-AI system with low misclassification error even when there exists a linear classifier and rejector that have zero error (the \emph{realizable} setting). 
 We prove that obtaining a linear pair with low error is NP-hard even when the problem is realizable. To complement this negative result, we give a mixed-integer-linear-programming (MILP) formulation that can optimally solve the problem in the linear setting. However, the MILP only scales to moderately-sized problems. 
 Therefore, we provide a novel surrogate loss function that is realizable-consistent and performs well empirically. We test our approaches on a comprehensive set of datasets and compare them to a wide range of baselines.
\end{abstract}

\section{Introduction}
\label{sec:intro}

AI systems are frequently used in combination with human decision-makers, including in high-stakes settings like healthcare \citep{beede2020human}. 
In these scenarios, machine learning predictors should be able to defer to a human expert instead of predicting difficult or unfamiliar examples.
However, when AI systems are used to provide a \emph{second opinion} to the human, prior work shows that humans over-rely on the AI when it is incorrect \citep{jacobs2021machine,mozannar2021teaching}, and these systems rarely achieve performance higher than either the human or AI alone \citep[Proposition 1]{liu2021understanding}.
This motivates deferral-style systems, where \emph{either} the classifier or the human predicts, to avoid over-reliance.

\begin{figure}[h]
    \centering
    \includegraphics[width=0.8\textwidth]{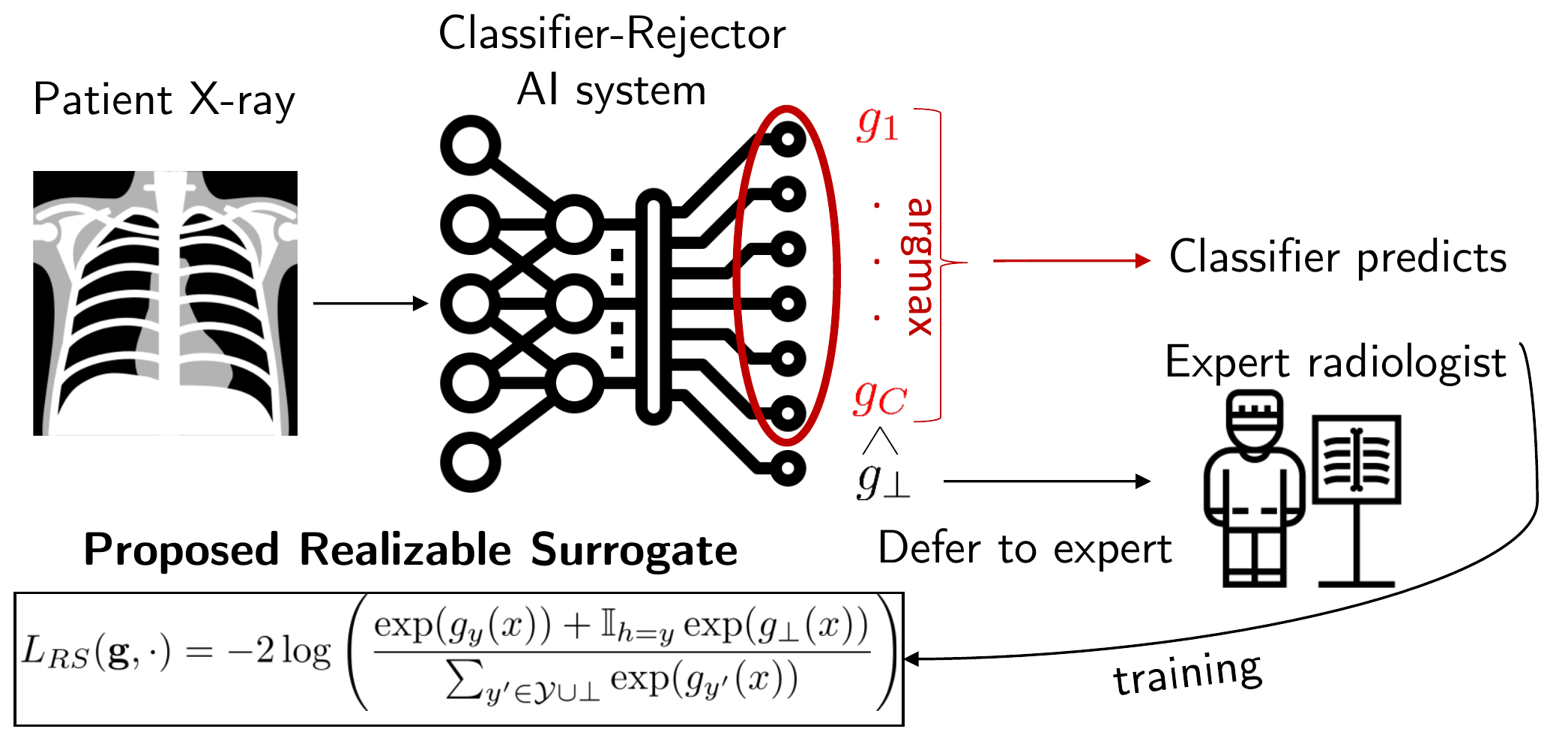}
    \caption{The learning to defer setting with the \realizablesurrogate illustrated in the application of making predictions for chest X-rays. }
    \label{fig:setting_realizable_surr}
\end{figure}

As a motivating example, suppose we want to build an AI system 
to predict the presence of pneumonia from a patient's chest X-ray, jointly with a  human radiologist. 
The goal of this work is to jointly learn a classifier that can predict pneumonia and a \textit{rejector}, which decides
on each data point whether the classifier or the
human should predict illustrated in Figure \ref{fig:setting_realizable_surr}. 
By learning the classifier jointly with the rejector, the aim is for the classifier to \emph{complement} the radiologist so that the Human-AI team performance is higher. 
We refer to the error rate of the Human-AI team as the \emph{system error}.

\textbf{Failure of Prior Approaches.} 
Existing literature has focused on \emph{surrogate loss functions} for deferral \citep{madras2018predict,mozannar2020consistent, verma2022calibrated} and confidence-based approaches \citep{raghu2019algorithmic,okati2021differentiable}. 
We give a simple synthetic setting where all of these approaches fail to find a classifier/rejector pair with a low system error.
In this setting, there exists a halfspace classifier and halfspace rejector that have zero system error (illustrated in Figure \ref{fig:expert_gauss}), but our experiments in Section \ref{sec:exp_synthetic} demonstrate that all prior approaches fail to find a good classifier/rejector pair in this setting.

To understand possible reasons for this failure, we first study the \textbf{computational complexity} of learning with deferral using halfspaces for the rejector and the classifier, which we call LWD-H. 
The computational complexity of learning with deferral has received little attention in the literature.
We prove that even in our simple setting where the data is realizable (i.e., there exists a halfspace classifier and halfspace rejector achieving zero system error), there is no polynomial-time algorithm that finds an approximately optimal pair of halfspaces unless $NP=RP$. 
We also extend our hardness result to approximation algorithms and when the data is not realizable by halfspaces.
In contrast, training a linear classifier in the realizable linear setting can be solved in polynomial time with linear programming \citep{boyd2004convex}.

Learning with deferral using halfspaces is also of significant \emph{practical} importance.
Sample efficiency is critical in learning with deferral since the training data is expensive to collect---it requires both human outputs  \emph{and} ground-truth labels. This motivates restricting to smaller model classes, and in particular to linear classifiers and rejectors. Linear models have the benefit of being interpretable with respect to the underlying features, which can be crucial for a human-AI deferral system. 
Additionally, the \emph{head tuning} or \emph{linear probing} paradigm, where only the final (linear) layer of a pretrained deep neural network is fine-tuned on different tasks, has become increasingly common as pretrained representations improve in quality, and it can be more robust than full fine-tuning \citep{kumar2021fine}.
However, as previously mentioned,  existing surrogate approaches fail to find a good  linear classifier and rejector even when one is guaranteed to exist. This motivates the need for an algorithm for {\bf exact minimization}
of the system training error.

We show that exact minimization of the system error can be formulated as a \textbf{mixed integer linear program} (MILP). 
This derivation overcomes a naive quadratic formulation of the problem. 
In addition to exactly minimizing the training loss, the MILP formulation allows us to easily incorporate constraints to govern the human-AI system. 
We show that modern commercial solvers such as Gurobi \citep{gurobi} are capable of solving fairly large instances of this MILP, making it a practical algorithm for the LWD-H problem.
To obtain similar gains over prior approaches, but with a more scalable algorithm, we develop a new differentiable surrogate loss function $L_{RS}$, dubbed \textbf{\realizablesurrogate}, that can  solve the LWD-H problem in the realizable setting by virtue of being \textit{realizable-consistent} \citep{long2013consistency} for a large class of hypothesis sets that includes halfspace classifier/rejector pairs.
We also show empirically that $L_{RS}$ is competitive with prior work in the non-linear setting.

In section \ref{sec:problem}, we formalize the learning with deferral problem.
We then study the computational complexity of LWD-H in section \ref{sec:linear_case}. We introduce our MILP approach in section \ref{sec:milp} and our new surrogate \textbf{\realizablesurrogatenosp} in section \ref{sec:new_method}. 
In section \ref{sec:experiments}, we evaluate our algorithms and baselines on a wide range of benchmarks in different domains, providing the most expansive evaluation of expert deferral algorithms to date.
To summarize, the contributions of this paper are the following:

\begin{itemize}
    \setlength\itemsep{0em}
    \item \textbf{Computational Complexity of Deferral:} We prove the computational hardness of PAC-learning with deferral in the linear setting.
    \item \textbf{Mixed Integer Linear Program Formulation and \realizablesurrogate:} We show how to formulate learning to defer with halfspaces as a MILP and provide a novel  surrogate loss.
    \item \textbf{Experimental Evaluation:} We showcase the performance of our algorithms on a wide array of datasets and compare them to several existing baselines. We contribute a publicly available repository with implementations of existing baselines and datasets: 
    \url{https://github.com/clinicalml/human_ai_deferral}
\end{itemize}

\section{Related Work}\label{sec:related_work}

A natural baseline for the learning to defer problem is to first learn a classifier that minimizes average misclassification error, then learn a model that predicts the probability that the human makes an error on a given example, and finally defer if the probability that the classifier makes an error is higher than that of the human. 
This is the approach adapted by \citet{raghu2019algorithmic}. 
However, this does not allow the classifier to adapt to the human.
Another natural approach is to model this problem as a mixture of experts: the human and the AI. 
This is the approach introduced by \citet{madras2018predict} and adapted by \citet{wilder2020learning,pradier2021preferential} by introducing a mixture of experts surrogates. However, this approach has been to shown to fail empirically as the loss is not easily amenable to optimization.
Subsequent work \citep{mozannar2020consistent} introduced consistent surrogate loss functions for the learning with deferral objective, with follow-up approaches addressing limitations including better calibration \citep{raman2021improving,liu2021incorporating}. 
Another consistent convex surrogate was proposed by \citet{verma2022calibrated} via a one-vs-all approach. 
\citet{charusaie2022sample} provides a family of convex surrogates for learning with deferral that generalizes prior approaches, however, our proposed surrogate does not belong to that family. 
\citet{keswani2021towards} proposes a surrogate loss  which is the sum of the loss of learning the classifier and rejector separately but that is not a consistent surrogate. 
\citet{de2020regression} proved the hardness of linear regression (not classification) where some training points are allocated to the human (not deferral but subset selection of the training data).
Finally, \citet{okati2021differentiable} proposes a method that iteratively optimizes the classifier on points where it outperforms the human on the training sample, and then learns a post-hoc rejector to predict who between the human and the AI has higher error on each point.  
\looseness=-1 
The setting when the cost of deferral is constant has a long history in machine learning and goes by the name of rejection learning  \citep{cortes2016learning, chow1970optimum,bartlett2008classification,charoenphakdee2021classification} or selective classification (only predict on $x$\% of data) \citep{el2010foundations,geifman2017selective,gangrade2021selective,acar2020budget}.  Our MILP formulation is inspired by work in binary linear classification that optimizes the 0-1 loss exactly \citep{ustun2016supersparse,nguyen2013algorithms}. 

\section{Learning with Deferral: Problem Setup}\label{sec:problem}

We frame the learning with deferral setting as the task of predicting a target $Y \in \mathcal{Y}=\{1,\cdots,C\}$. 
The classifier has access to features $X \in \mathcal{X} = \mathbb{R}^d$, while the human (also referred to as the expert) has access to a potentially different set of features $Z \in \mathcal{Z}$ which may include side-information beyond $X$.
The human 
is modeled as a fixed predictor $h: \mathcal{Z} \to \mathcal{Y}$. 
The AI system consists of a classifier $m: \mathcal{X} \to \mathcal{Y}$ and a rejector $r: \mathcal{X} \to \{0,1\}$. 
When $r(x)=1$, the prediction is deferred to the human and we incur a cost if the human makes an error, plus an additional, optional penalty term: $\ell_{\textrm{HUM}}(x,y,h) = \bI_{h \neq y} + c_{\textrm{HUM}}(x,y,h)$. When $r(x)=0$, then the classifier makes the final decision and incurs a cost with a different 
optional penalty term: $\ell_{\textrm{AI}}(x,y,m) = \bI_{m \neq y} + c_{\textrm{AI}}(x,y,m)$.
We can put this together into a loss function for the classifier and rejector:
\begin{align}
   \nonumber &L_{\mathrm{def}}(m,r)=  \bE_{X, Y, Z} \ [ \ \ell_{\textrm{AI}}\big(X,Y,m(X)\big) \cdot \bI_{r(X) = 0}  +  \ell_{\textrm{HUM}}(X,Y,h(Z)) \cdot \bI_{r(X)=1} \  ].  
\end{align}
In this paper we focus mostly on the cost of misclassification with no additional penalties, the deferral loss becomes a misclassification loss $L_{\mathrm{def}}^{0{-}1}(m,r)$ for the human-AI system, and the optimization problem is:
\begin{align}
    &\minimize_{m, r} L_{\mathrm{def}}^{0{-}1}(m,r) := \P \left[ \left( \left(1-r(X) \right)m(X) + r(X) h(Z) \right) \ne Y \right].
    \label{eq:01_reject_loss}    
\end{align}

\textbf{Constraints.}
We may wish to constrain the behavior of the human-AI team when learning the classifier and rejector pair. 
For example, we may have a limit on the percentage of times the AI can defer to the human, due to the limited time the human may have. 
We express this as a \textbf{coverage} constraint:
\begin{equation}\label{eq:coverage}
    \bP(r(X) = 1 ) \leq \beta.
\end{equation}

Finally, it is desirable 
that our system does not perform differently 
across different demographic groups. 
Let $A \in \{1,\cdots,|A|\}$ denote the demographic identity of an individual. 
Then if we wish to equalize the error rate across demographic groups, we impose the \textbf{fairness} constraint $\forall a \in A$:
\begin{align}\label{eq:fairness}
      \nonumber &\bP( (1-r(X))m(X) + r(X) h(Z) \ne Y |A\neq a) \\&= \bP( (1-r(X))m(X) + r(X) h(Z) \ne Y |A=a)
\end{align}

\textbf{Data.} We assume access to samples $S=\{(x_i,h(z_i), y_i)\}_{i=1}^n$ where $h(z_i)$ is the human's prediction on the example, but note that we do not observe $z_i$, the information used by the human. 
We emphasize that the label $y_i$ and human prediction $h(z_i)$ are different, even though $y_i$ could also come from humans. For example in our chest X-ray classification example, $y_i$ could come from a consensus of 3 or more radiologists, while $h(z_i)$ is the prediction of a single radiologist not involved with the label.
Given the dataset $S$ the 
 system \emph{training loss} is given by:
\begin{equation}\label{eq:training_loss}
\sysLhat(m,r) := \frac{1}{n}\sum_{i=1}^n\bI_{m(x_i) \neq y_i} \bI_{r(x_i) = 0} +  \bI_{h(z_i) \neq y_i}\bI_{r(x_i)=1}
\end{equation}

In the following section, we study the computational complexity of learning with deferral using halfspaces, which reduces to studying the optimization problem \eqref{eq:training_loss} when $m$ and $r$ are constrained to be halfspaces.

\section{Computational Complexity of Learning with Deferral }\label{sec:linear_case} 

\begin{figure*}[h]
\centering

  \includegraphics[clip,width=\textwidth,trim={0.0cm 9.3cm 9.0cm 0.0cm}]{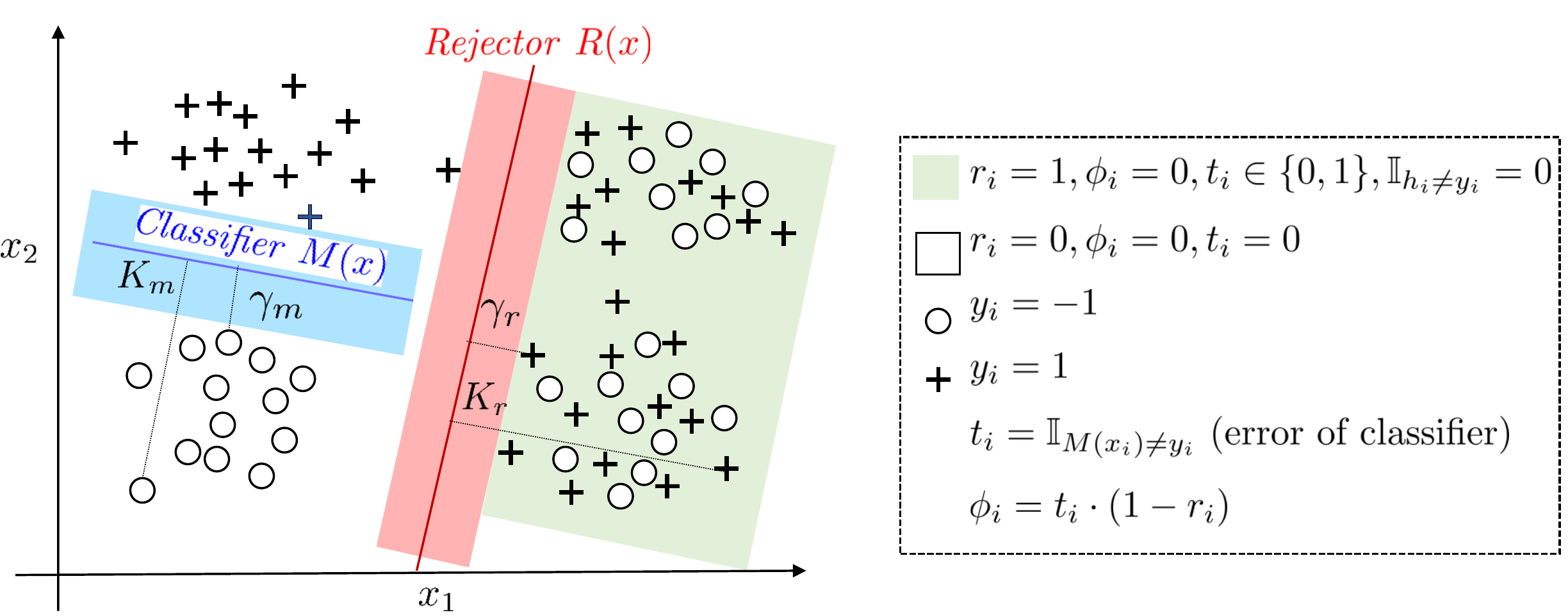}
  
    \caption{ The realizable LWD-H setting illustrated. The task is binary classification with labels $\{o,+\}$; the human is perfect on the green-shaded region, and the data outside the green region is linearly separable. As a result, the optimal classifier and rejector obtain zero error. Assumption \ref{ass: margin} is illustrated graphically as well as the MILP variables of equations \eqref{eq:milp_obj}-\eqref{eq:mil_last_cst}.}
    \label{fig:expert_gauss}
\end{figure*}

The misclassification error of the human-AI team in equation \eqref{eq:01_reject_loss} is challenging to optimize as it requires searching over a joint set of functions for the classifier and rejector, in addition to dealing with the nonconvex $0$-$1$ aspect.
To study the computational complexity of minimizing the loss, we restrict our attention to a foundational setting: linear classifiers and linear rejectors in the binary label scenario. 

We begin with the realizable case when there  exists a halfspace classifier and rejector that can achieve zero loss:
\begin{assumption}[Realizable Linear Setting]\label{ass:realizability}
Let $\mathcal{X} = \mathbb{R}^d$ and $\mathcal{Y} = \{0,1\}$. We assume that for the given expert $h$ there exists a linear classifier $m^*(x) = \bI_{M^\top x > 0}$ and a linear rejector $r^*(x) = \bI_{ R^\top x >0}$ that achieve 0 error: 
\begin{equation*}
    \bE_{(x,y,z)\sim \mathbf{P}} \ [ \ \bI_{m^*(x) \neq y} \bI_{r^*(x) = 0} +  \bI_{h(z) \neq y}\bI_{r^*(x)=1} \  ] = 0.
\end{equation*}
\end{assumption}
This setting is illustrated in Figure  \ref{fig:expert_gauss}. Since the decision regions of $m$ and $r$ are halfspaces, we also use the term ``halfspace'' interchangeably.
Note that while the classifier is assumed to be linear, the human can have a non-linear decision boundary. The analog of this assumption in the binary classification without deferral setting is to assume that there exists a halfspace that can correctly classify all the data points.
In that case, we can formulate the optimization problem as a linear program to efficiently find the optimal solution \citep{boyd2004convex}.

\textbf{Hardness.} \looseness=-1 In contrast to learning without deferral, 
we will prove that in general, it is computationally hard to learn a linear $m$ and $r$ under Assumption \ref{ass:realizability}. 
Define the \emph{learning with deferral using halfspaces} (LWD-H) problem as that of finding halfspace $m$ and halfspace $r$ such that the system error in \eqref{eq:01_reject_loss} is approximately minimized. 
\begin{theorem} \label{prop:hardness}
Let $\epsilon > 0$ be an arbitrarily small constant.
Under a guarantee that there exist halfspaces $m^*$, $r^*$ with zero system loss (Assumption \ref{ass:realizability}), there is no polynomial-time algorithm to find a pair of classifier-rejector halfspaces with error $1/2 -\epsilon$ unless $NP=RP$.
\end{theorem}
This shows that even in the realizable setting (i.e., there exists a pair of halfspaces with zero system loss), it is hard to find a pair of halfspaces that even gets system error $1/2-\epsilon$.
\begin{corollary }
    There is no efficient proper PAC-learner for LWD-H unless $NP=RP$.
\end{corollary }
\begin{sproof}
First, because the true distribution of points could be supported on a finite set, the LWD-H problem boils down to approximately minimizing the training loss \eqref{eq:training_loss}.
Our proof relies on a reduction from the problem of learning an intersection of two halfspaces in the realizable setting.
Let $D=\{x_i,y_i\}_{i=1}^n$ and suppose there exists an intersection of two half-spaces $g_1,g_2$ that achieve 0 error for $D$.
This is an instance of learning an intersection of two halfspaces in the realizable setting, which is hard to even \emph{weakly} learn \citep{khot2011hardness}.
We show that this is an instance of the realizable LWD-H problem by setting $m = g_1$ and $r=\bar{g_2}$ and the human $H$ to always predict 0. 
Hence, an algorithm for efficiently finding a classifier/rejector pair with error $\frac{1}{2}-\epsilon$ would also find an intersection of halfspaces with error $\frac{1}{2} - \epsilon$, which is hard unless $NP=RP$.
\end{sproof}

All proofs can be found in the Appendix. 
This hardness result holds in the realizable setting, with proper learning, and with no further assumptions on the data distribution. 

\textbf{Extensions.} Even if the problem is not realizable and the goal is to find an approximation algorithm, this is still computationally hard as presented in the following corollary. 

\begin{corollary }\label{th:hardness_extension}
When the data is not realizable (i.e., Assumption \ref{ass:realizability} is violated), there is no polytime algorithm for finding a pair of halfspaces with error $\frac{1}{2} - \epsilon$ unless $NP=RP$.
\end{corollary }

\textbf{Exact Solution}. These hardness results motivate the need for new approaches to solving the LWD-H problem. 
In the next section, we derive a scheme to exactly minimize the misclassification error of the human/AI system using mixed-integer linear programming (MILP).

\section{Mixed Integer Linear Program Formulation}\label{sec:milp}
In the previous section, we saw that in the linear setting it is computationally hard to learn an optimal classifier and rejector pair.
As discussed in Section \ref{sec:intro}, we are interested in the linear setting due to the cost of labeling large datasets for learning with deferral. 
Linear predictors can perform similar to non-linear predictors in applications involving high-dimensional medical data \citep{razavian2015population}. Moreover, we can rely on pre-trained representations, which can allow 
linear predictors on top of embedded representations to attain performance comparable to non-linear predictors \citep{bengio2013representation}. 

\textbf{A First Formulation.} As a first step, we write down a mixed integer \emph{nonlinear} program for the optimization of the training loss $\sysLhat$ in \eqref{eq:training_loss} over linear classifiers and linear rejectors with binary labels. 
For simplicity, let $\mathcal{Y} = \{-1,+1\}$. A direct translation of \eqref{eq:training_loss} with halfspace classifiers and rejectors yields the following:
\begin{align}
&M^*, R^*, \cdot = \altargmin_{M,R,m_i,r_i} \sum_{i=1}^n (1-r_i) \bI_{m_i \neq y_i}  + r_i \bI_{h_i \neq y_i}   \\
\label{eq:constraint_trivial_mip}
& \text{s.t.} \quad m_i = \textrm{sign}(M^\top x_i), \  \ r_i  =  \bI_{R^\top x_i \geq 0}  \ 
 \forall i \in [n], \\  & \quad M \in \mathbb{R}^d, R \in \mathbb{R}^d. \nonumber
\end{align}
The variables $m_i$ and $r_i$ are simply the binary outputs of the classifier and rejector. We observe that the objective involves a quadratic interaction between the classifier and rejector. Furthermore, the constraints \eqref{eq:constraint_trivial_mip} are indicator constraints that are difficult to optimize.

\textbf{Making Objective Linear.} We observe that since the $r_i$'s are binary, the term $(1-r_i) \bI_{m_i \neq y_i}$ can be equivalently rewritten as $\max(0,\bI_{m_i \neq y_i} - r_i)$. 
This is a crucial simplification that avoids having a mixed integer quadratic program. Below we use this to create a binary variable $t_i = \bI_{m_i \neq y_i}$ representing the error of the classifier and a second \emph{continuous} variable $\phi_i$ that upper bounds $\max(0,t_i - r_i)$ and represents the classifier error after accounting for deferral. 

\textbf{Making Constraints Linear.} Constraints \eqref{eq:constraint_trivial_mip} make sure that the binary variables $r_i$ and $m_i$ are the predictions of half-spaces $R$ and $M$ respectively. As mentioned above, we will formulate the problem using the classifier error variables $t_i$ instead of the classifier predictions $m_i$. To reformulate constraints \eqref{eq:constraint_trivial_mip} in a linear fashion, we have to make assumptions on the optimal $M$ and $R$:
\begin{assumption}[Margin] \label{ass: margin}
The optimal solution $(M,R)$ that minimizes the training loss \eqref{eq:training_loss} has margin and is bounded, meaning that $(M, R)$ satisfy the following for all $i \in [n]$ in the training set and some $\gamma_m, \gamma_r, K_m,$ $K_r > 0$:
\begin{equation}\label{eq:margin_norm_bounds}
\resizebox{0.43\textwidth}{!}{
    $\gamma_m \leq |M^\top x_i| \leq K_m - \gamma_m, \  \gamma_r \leq |R^\top x_i| \leq K_r - \gamma_r$}
\end{equation}
\end{assumption}
A similar assumption is made in \citep{ustun2016supersparse}. The upper bounds in \eqref{eq:margin_norm_bounds} are often naturally satisfied as we usually deal with bounded feature sets $\mathcal{X}$ such that we can normalize $x_i$ to have unit norm, and the 
norms of $M$ and $R$ are constrained for regularization. 

\textbf{Mixed Integer Linear Program.} With the above ingredients and taking inspiration from the big-M approach of \citet{ustun2016supersparse}, we can write down the resulting mixed integer linear program (MILP):
\begingroup
\allowdisplaybreaks

\begin{empheq}[box=\mymath]{align} 
\nonumber &M^*, R^*, \dots =\\& \altargmin_{M,R,\{r_i\},\{t_i\},\{\phi_i\}} \sum_{i} \phi_i + r_i \bI_{h_i \neq y_i}, \ s.t. \\
& \phi_i \geq t_i - r_i, \qquad \phi_i \geq 0 \quad  \forall i \in [n] \label{eq:milp_obj}\\
& K_m t_i \geq \gamma_m - y_i M^\top x_i \quad \forall i \in [n] \label{eq:H_constraint_milp_bin}\\
& R^{\top} x_i \leq K_r r_i + \gamma_r (r_i - 1) , \\&  R^{\top} x_i \geq K_r (r_i - 1) + \gamma_r r_i \label{eqn:MILP_RLB} \quad \forall i \in [n]\\
& r_i \in \{0,1\}, t_i \in \{0,1\}, \\& \phi_i \in \mathbb{R}^+ \quad \forall i \in [n],
\ M,R \in \mathbb{R}^d \label{eq:mil_last_cst}
\end{empheq}
\endgroup

Please see Figure \ref{fig:expert_gauss} for a graphical illustration of the variables.
We show that constraints \eqref{eqn:MILP_RLB} function as intended; the rest of the constraints are verified in the Appendix.  When $r_i=0$, then we have the constraints $ R^{\top} x_i  \leq - \gamma_r$ and $R^{\top} x_i \geq -K_r$: this correctly forces the rejector to be negative. When $r_i=1$, we have $R^{\top} x_i  \geq  \gamma_r$ and $R^{\top} x_i \leq K_r$: which means the rejector is positive. 
Note that we do not need to know the margin $\gamma_r$ exactly, only a lower bound $\gamma$, $0<\gamma \leq \gamma_r$; the formulation is still correct with $\gamma$ in place of $\gamma_r$. However, we cannot set $\gamma = 0$ as then the trivial solution $R=0$ is feasible and the constraint is void. The same statements apply to $\gamma_m$. This MILP has $2n$ binary variables, $n + 2d$ continuous variables and $4n$ constraints. Finally, note that the MILP minimizes the 0-1 error even when Assumption \ref{ass:realizability} is violated.

\textbf{Regularization and Extension to Multiclass.} We can add $l_1$ regularization to our model by adding the $l_1$ norm of both $M$ and $R$ to the objective. This is done by defining two sets of variables constrained to be the $l_1$ norm of the classifier and rejector and adding their values to the objective in \eqref{eq:milp_obj}. Adding regularization can help prevent the MILP solution from overfitting to the training data. The above MILP only applies to binary labels but can be generalized to the multi-class setting where $\mathcal{Y} = \{1,\cdots, C\}$ (see Appendix).


\textbf{Generalization Bound.} Under Assumption \ref{ass: margin} and non-realizability, assume $\lVert x_i \rVert_1 \leq 1$ and constrain the search of the MILP to $M$ and $R$ with infinity norms of at most $K_m$ and $K_r$ respectively. We can relate the performance of MILP solution on the training set to the population 0-1 error.
\begin{proposition}
For any expert $h$ and data distribution $\mathbf{P}$ over $\mathcal{X} \times \mathcal{Y}$ that satisfies Assumption \ref{ass: margin}, let $0<\delta<\frac{1}{2}$. Then  with probability at least $1-\delta$, the following holds for the empirical minimizers $(\hat{m}^*,\hat{r}^*)$ obtained by the MILP:
\begin{align*}
    &L_{0{-}1}(\hat{m}^*,\hat{r}^*) \leq \sysLhat(\hat{m}^*,\hat{r}^*) \\&+  \frac{(K_m + K_r) d \sqrt{2 \log d} + 10 \sqrt{\log(2/\delta)}}{\sqrt{n \bP(h(Z) \neq Y)}}.    
\end{align*}
\end{proposition}
This bound improves on surrogate optimization since the MILP will achieve a lower training error, $\sysLhat(\hat{m}^*,\hat{r}^*)$, than the surrogate, which optimizes a different objective. 

\textbf{Adding Constraints.} A major advantage of the MILP formulation is that it allows us to provably integrate any linear 
constraints on the variables with ease. For example, the constraints mentioned in Section \ref{sec:problem} can be added to the MILP as follows in a single constraint:
\begin{itemize}
    \item Coverage: $\sum_i r_i /n \leq \beta$ 
    \item Fairness: $\sum_{i: A=1} \bigl(\phi_i + r_i \bI_{h_i \neq y_i}\bigr) / \lvert\{i: A=1\}\rvert = \sum_{i: A=0} \bigl(\phi_i + r_i \bI_{h_i \neq y_i}\bigr) / \lvert\{i: A=0\}\rvert $.
\end{itemize}

So far, we have provided an exact solution to the linear learning to defer problem.
However, the MILP requires significant computational time to find an exact solution for large datasets.
Moreover, we might need a non-linear classifier or rejector to achieve good error.
The remaining questions are (i) how to efficiently find a good pair of halfspaces for large datasets and (ii) how to generalize to non-linear predictors. 
In the following section, we give a novel surrogate loss function that is optimal in the realizable LWD-H setting, performs well with non-linear predictors, and can be efficiently minimized (to a local optimum). 


\section{Realizable Consistent Surrogate}\label{sec:new_method}
\subsection{Consistency vs Realizable Consistency}
\looseness=-1 Most machine learning practice is based on optimizing surrogate loss functions of the true loss that one cares about. 
The surrogate loss functions are chosen so that optimizing them also optimizes the true loss functions, and also chosen to be differentiable so that they are readily optimized.
This first property is captured by the notion of \emph{consistency}, which was the main focus of much of the prior work on expert deferral: \citep{mozannar2020consistent,verma2022calibrated,charusaie2022sample}. 
We start by giving a formal definition of the consistency of a surrogate loss function:

\begin{definition}[Consistency\footnote{This is also referred to as Fisher consistency \citep{lin2002support} and classification-calibration \citep{bartlett2006convexity}.}]
A surrogate loss function $\Tilde{L}(m,r)$ is a consistent loss function for another loss $L_{\mathrm{def}}^{0{-}1}(m,r)$ if optimizing the surrogate over all measurable functions is equivalent to minimizing the original loss.

\end{definition}
For example, the surrogates $L_{CE}$ and $\Psi_{OvA}$ both satisfy consistency for $L_{\mathrm{def}}^{0{-}1}(m,r)$ \citep{mozannar2020consistent,verma2022calibrated}. It is crucial to note that consistency only applies when optimizing over \emph{all measurable functions}. 
Conversely, in LWD-H, and in the setting of Figure \ref{fig:expert_gauss}, when we optimize with linear functions, consistency does not provide any guarantees, which explains why these methods can fail in that setting.

Since we normally optimize over a restricted model class, we want our guarantee for the surrogate to also hold for optimization under a certain model class. The notion of realizable $\cH$-consistency is such a guarantee that has proven fruitful for classification \citep{long2013consistency,zhang2020bayes} and was extended by \citet{mozannar2020consistent} for learning with deferral. We recall the notion when extended for learning with deferral:

\begin{definition}[realizable $(\mathcal{M},\mathcal{R})$-consistency]
A surrogate loss function $\Tilde{L}(m,r)$ is a realizable  $(\mathcal{M},\mathcal{R})$-consistent loss function for the loss $L_{\mathrm{def}}^{0{-}1}(m,r)$ if there exists a zero error solution $m^*,r^* \in \cM \times \cR$ with $L_{\mathrm{def}}^{0{-}1}(m^*,r^*)=0$. Then optimizing the surrogate returns such zero error solution:
\begin{equation*}
    \tilde{m},\tilde{r} \in \arginf_{m,r \in \cM \times \cR  } \Tilde{L}(m,r) \implies L_{\mathrm{def}}^{0{-}1}(\tilde{m},\tilde{r})=0
\end{equation*}
\end{definition}
Realizable $(\cM, \cR)$-consistency guarantees that when our data comes from some ground-truth $m^*, r^* \in \cM \times \cR$, then minimizing the (population) surrogate loss will find an optimal $(m,r)$ pair.
We propose a novel, differentiable, and $(\cM, \cR)$-consistent surrogate for learning with deferral when $\cM$ and $\cR$ are \textit{closed under scaling}.
A  class $\cG$ of scoring functions from $\mathcal{X}$ to $\mathbb{R}^C$ is closed under scaling if $\bm{g} \in \cG \implies \alpha \bm{g} \in \cG$ for any $\alpha \in \mathbb{R}$.
For example, we can let $\cG$ be the class of linear scoring functions $\bm{g}(x) = G^{\top}x$ and $G \in \mathbb{R}^{ d \times C}$.
Our results hold for arbitrary $\cG$ that are closed under scaling, e.g., ReLU feedforward neural networks (FNN).
We parameterize the $(m,r)$ pair with  $|\cY|+1$ dimensional scoring function $\bm{g}:(g_1,\ldots, g_{|\cY|}, g_\bot)$.
We define $m(x) = \arg \max_{y \in \mathcal{Y}}g_y(x)$ and $r(x)= \bI_{\max_{y \in \mathcal{Y}}g_y(x) \leq g_\bot(x) }$. 
The joint classifier-rejector model class $(\mathcal{M},\mathcal{R})$ is thus defined by $\cG$, and we say $(\mathcal{M},\mathcal{R})$ is closed under scaling whenever $\cG$ is closed under scaling. 
The proposed new surrogate loss $L_{RS}$, dubbed \emph{\realizablesurrogatenosp}, is defined at each point $(x,y,h)$ as:
  
\begin{empheq}[box=\mymath]{equation}
 L_{RS}(\mathbf{g},\cdot) = -2 \log\left(\frac{\exp(g_{y}(x)) + \bI_{h = y} \exp(g_{\bot}(x)) }{\sum_{y' \in \mathcal{Y} \cup  \bot}\exp(g_{y'}(x))} \right) \label{eq:proposed_loss_lrs}
\end{empheq}

\subsection{Derivation of Surrogate}
We now derive our proposed surrogate  \emph{\realizablesurrogatenosp} with a principled approach.
In this paper, our primary goal is to predict a target $Y$ given a set of covariates $X$ while having the ability to query a human $H$ to predict. Our overall predictor is denoted as $\hat{Y}$, a function of both $H$ and $X$,  our  goal is  learning a predictor that maximizes the agreement between $\hat{Y}$ and $Y$:
\begin{equation*} \E[\bI_{\hat{Y}(X,H)=Y}] = 
  \E_{X}\left[ \bP(\hat{Y}(x,H)=Y|X=x)  \right]
\end{equation*}
It will be easier to maximize the logarithm of the probability and thus using Jensen's inequality we obtain an upper  bound :
\begin{equation*}
    \log \left( \E_{X}\left[ \bP(\hat{Y}(x,H)=Y|X=x)  \right] \right)
 \leq \E_{X}\left[ \log\left(\bP(\hat{Y}(x,H)=Y|X=x) \right) \right]
\end{equation*}
We now decompose our predictor into a classifier-rejector pair $(m,r)$ where the rejector decides if the classifier or the human should predict. This transforms our objective to:

\begin{align}
&L = \E_{X}\left[ \log\left(\bP(\hat{Y}=Y|X=x) \right) \right] =  \nonumber \\&  \E_{X}\left[ \log\left(\bP(\hat{Y}=Y|X=x,r(x)=0) \bP(r(x)=0|X=x) + \bP(\hat{Y}=Y|X,r(x)=1) \bP(r(x)=1|X=x)   \right) \right]   \nonumber \\
&= \E_{X}\left[ \log\left(\bP(m(x)=Y|X=x) \bP(r(x)=0|X=x) + \bP(H=Y|X=x) \bP(r(x)=1|X=x)   \right) \right] \label{eq:L_decompose}
\end{align}

In \cite{madras2018predict}, their proposed loss splits the sum inside the log above into a sum of log-likelihoods of the classifier and expert each weighted by the probability of predicting and deferring respectively.   Instead in this work, we try to optimize the above log likelihood $L$ \eqref{eq:L_decompose} directly.

\paragraph{Parameterization.} We now try to understand how we can parameterize the classifier-rejector pair. We first parameterize the classifier with a set of scoring functions
$g_y: \mathcal{X} \to \mathbb{R}$ for $y \in \mathcal{Y}$ and define the classifier as the label $y$ that attains the maximum value among the set $\{g_{y'}\}_{y' \in \mathcal{Y}}$.
To parameterize the rejector $r$, we define a single scoring function $g_{\bot}: \mathcal{X} \to \mathbb{R}$ and defer if $g_\bot > \max_{y} g_y$ which induces a comparison between the function $g_\bot$ and the classifier scores. One could instead parameterize the rejector $r$ with a single function $g_{\bot}: \mathcal{X} \to \mathbb{R}$ and defer if $g_\bot(x)$ is positive, we find empirically that the previous parameterization has better performance \footnote{This parameterization form can achieve a halfspace rejector and results in the following loss: $$ L_{RS2} = \E_{X}\left[ \log\left(  \frac{\exp(g_{Y}(x))}{\sum_{y' \in \mathcal{Y}} \cdot \exp(g_{y'}(x))}  \frac{1}{1+\exp(g_{\bot})}+ \bP(H=Y|X) \cdot \frac{\exp(g_{\bot})}{1+\exp(g_{\bot})} \right) \right]$$ }.

However, with  the characterization, both our classifier and rejector are deterministic. 
Plugging in the parameterization of $(m,r)$ into the loss in \eqref{eq:L_decompose} would result in a loss function that is non-differentiable in the parameters $g_y$ and $g_\bot$ due to 
 thresholding. Instead, we allow the classifier and rejector to only be probabilistic during training by defining:
 \[
\bP(m(x)=Y|X=x) =   \frac{\exp(g_{Y}(x))}{\sum_{y' \in \mathcal{Y}} \exp(g_{y'}(x))}, \quad \bP(r(x)=1|X=x) =  \frac{\exp(g_\bot(x))}{\sum_{i \in \mathcal{Y} \cup \bot}  \exp(g_{i}(x))}
  \]
This transforms the log liklelihood $L$ to:

\begin{align*}
 & \E_{X}\left[ \log\left(  \frac{\exp(g_{Y}(x))}{\sum_{y' \in \mathcal{Y}}  \exp(g_{y'}(x))}  \cdot \frac{\sum_{y' \in \mathcal{Y}}  \exp(g_{y'}(x))}{\sum_{i \in \mathcal{Y} \cup \bot}  \exp(g_{i}(x))}+ \bP(H=Y|X)  \cdot \frac{\exp(g_\bot(x))}{\sum_{i \in \mathcal{Y} \cup \bot}  \exp(g_{i}(x))} \right) \right] \\
  &= -\E_{X}\left[ \log\left(  \frac{\exp(g_{Y}(x)) + \bP(H=Y|X) \exp(g_\bot(x)) }{\sum_{y' \in \mathcal{Y}}  \exp(g_{y'}(x))}  \right) \right]
\end{align*}
We multiple the above likelihood by $-2$ so that we can instead minimize a loss and so that it becomes an upper bound of the $0-1$ deferral loss $L_{\mathrm{def}}^{0{-}1}(m,r)$.
Given the dataset $S$, our proposed loss then becomes:
\begin{equation}
L_{RS}^{S} = -2 \sum_{i=1}^n  \log\left(  \frac{\exp(g_{y_i}(x_i)) + \bI_{h(z_i)=y_i} \exp(g_\bot(x)) }{\sum_{y' \in \mathcal{Y}}  \exp(g_{y'}(x))}  \right)
\end{equation}

\subsection{Theoretical Guarantees}
\looseness=-1 Notice in our proposed loss $L_{RS}$ that when the human is incorrect, i.e. $\bI_{h=y}=0$, the loss incentivizes the classifier to be correct, similar to cross entropy loss. 
However, when the human is correct, the learner has the \emph{choice} to either fit the target or defer: there is no penalty for choosing to do one or the other.
This is what enables the classifier to complement the human and differentiates $L_{RS}$ from prior surrogates, such as $L_{CE}$ \citep{mozannar2020consistent}, that are not realizable-consistent (see Theorem \ref{apx:the_realiza_not} in Appendix \ref{apx:proffs_sec_realiz}) and penalize the learner for not fitting the target even when deferring. 
This property is showcased by the fact that our surrogate is realizable $(\mathcal{M},\mathcal{R})$-consistent  for model classes that are closed under scaling. 
Moreover, it is an upper bound of the true loss $L_{\mathrm{def}}^{0{-}1}(m,r)$.
The theorem below characterizes the properties of our novel surrogate function. 
\vspace{-0.13em}
\begin{theorem}
\label{thm:realizable-consistent}
The \emph{\realizablesurrogatenosp} $L_{RS}$ is a  realizable $(\mathcal{M},\mathcal{R})$-consistent  surrogate for $L_{\mathrm{def}}^{0{-}1}$ for model classes closed under scaling, and satisfies $L_{\mathrm{def}}^{0{-}1}(m,r) \le L_{RS}(m,r)$ for all $(m,r)$.
\end{theorem}
This theorem implies that when Assumption \ref{ass:realizability} is satisfied and $\cG$ is the class of linear scoring functions, minimizing $L_{RS}$ yields a classifier-rejector pair with zero system error.
The resulting classifier is the halfspace $\bI((G_1-G_0)^\top x \geq 0)$ and the form of the rejector is $\bI((G_{\bot}^\top x- \max(G_1^\top x,G_0^\top x)) \geq 0)$, which is an intersection of halfspaces. One can obtain a halfspace rejector by minimizing instead with the parameterization of $L_{RS2}$.

\looseness=-1 The surrogate is differentiable but \emph{non}-convex in $\mathbf{g}$, though it is convex in each $g_i$. 
Indeed, a jointly convex surrogate that provably works in the realizable linear setting would contradict Theorem \ref{prop:hardness}.
In practice, we observe that in the linear realizable setting, the local minima reached by gradient descent obtain zero training error despite the nonconvexity.
The mixture-of-experts surrogate in \cite{madras2018predict} is realizable $(\mathcal{M},\mathcal{R})$-consistent, non-convex and not classification consistent as shown by \cite{mozannar2020consistent}, however, \cite{mozannar2020consistent} also showed that it leads to worse empirical results than simple baselines. We have not been able to prove or disprove that \realizablesurrogatenosp is classification-consistent, unlike other surrogates like that of \citet{mozannar2020consistent}. It remains an open problem to find both a consistent \textbf{and} a realizable-consistent surrogate. 

\subsection{Underfitting The Target}
 Minimizing the proposed loss leads to a classifier that attempts to complement the human. 
One consequence is that the classifier might have high error on points that are deferred to the human, resulting in possibly high error across a large subset of the data domain. 
We can explicitly encourage the classifier to fit the target on all points by adding an extra term to the loss:
\begin{empheq}[box=\mymath]{align} \label{eq:proposed_RS_loss+extemsopm}
&     \resizebox{0.48\textwidth}{!}{$L_{RS}^{\alpha}(\mathbf{g},x,y,h) = - \alpha \log\left(\frac{\exp(g_{y}(x) + \bI_{h = y} \exp(g_{\bot}(x)) }{\sum_{y' \in \mathcal{Y} \cup  \bot}\exp(g_{y'}(x))} \right)$}  - (1-\alpha)  \log\left(\frac{\exp(g_{y}(x)) }{\sum_{y' \in \mathcal{Y} }\exp(g_{y'}(x))} \right)
\end{empheq}
The new loss $L_{RS}^{\alpha}$ with $\alpha \in [0,1]$ (a hyperparameter) is a convex combination of $L_{RS}$ and the cross entropy loss for the classifier (with the softmax applied only over the functions $g_y$ rather than including $g_{\bot}$).
Empirically, this allows the points that are deferred to the human to still help provide extra training signal to the classifier, which is useful for sample-efficiency when training complex, non-linear hypotheses.
Finally, due to adding the parameter $\alpha$, the loss no longer remains realizable consistent, thus we let the rejector be $r(x) = \bI_{g_{\bot}(x) - \max_y g_y(x) \geq \tau}$ and we learn $\tau$ with a line search to maximize system accuracy on a validation set.
In the next section, we evaluate our approaches with an extensive empirical benchmark.

\section{Experiments}\label{sec:experiments}

\newpage
\begin{figure}[H]
    \begin{subfigure}{0.5\textwidth}
        \centering
          \includegraphics[width=\textwidth]{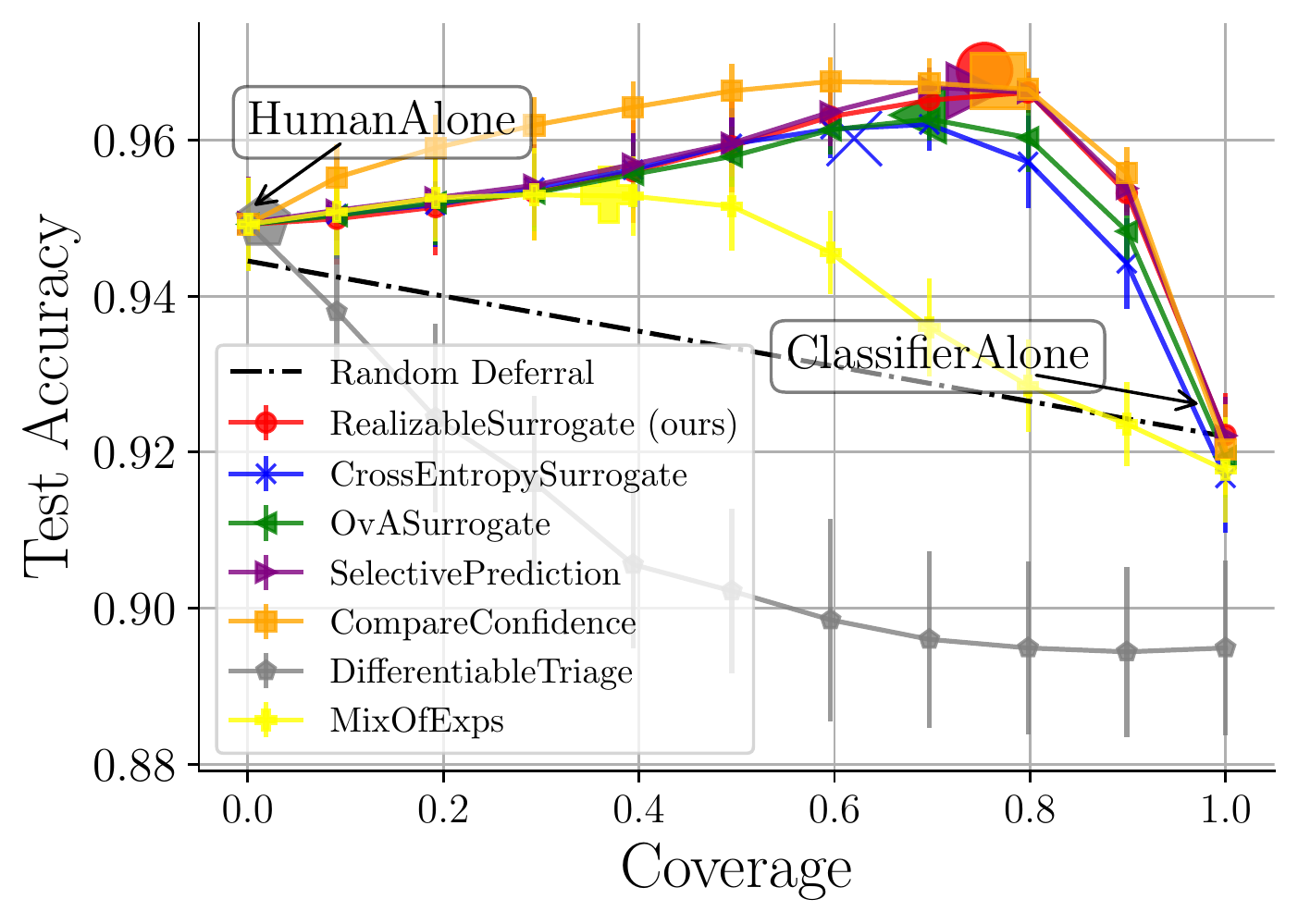}
        \subcaption{CIFAR-10H}
    \end{subfigure}%
    \begin{subfigure}{0.5\textwidth}
        \centering
          \includegraphics[width=\textwidth]{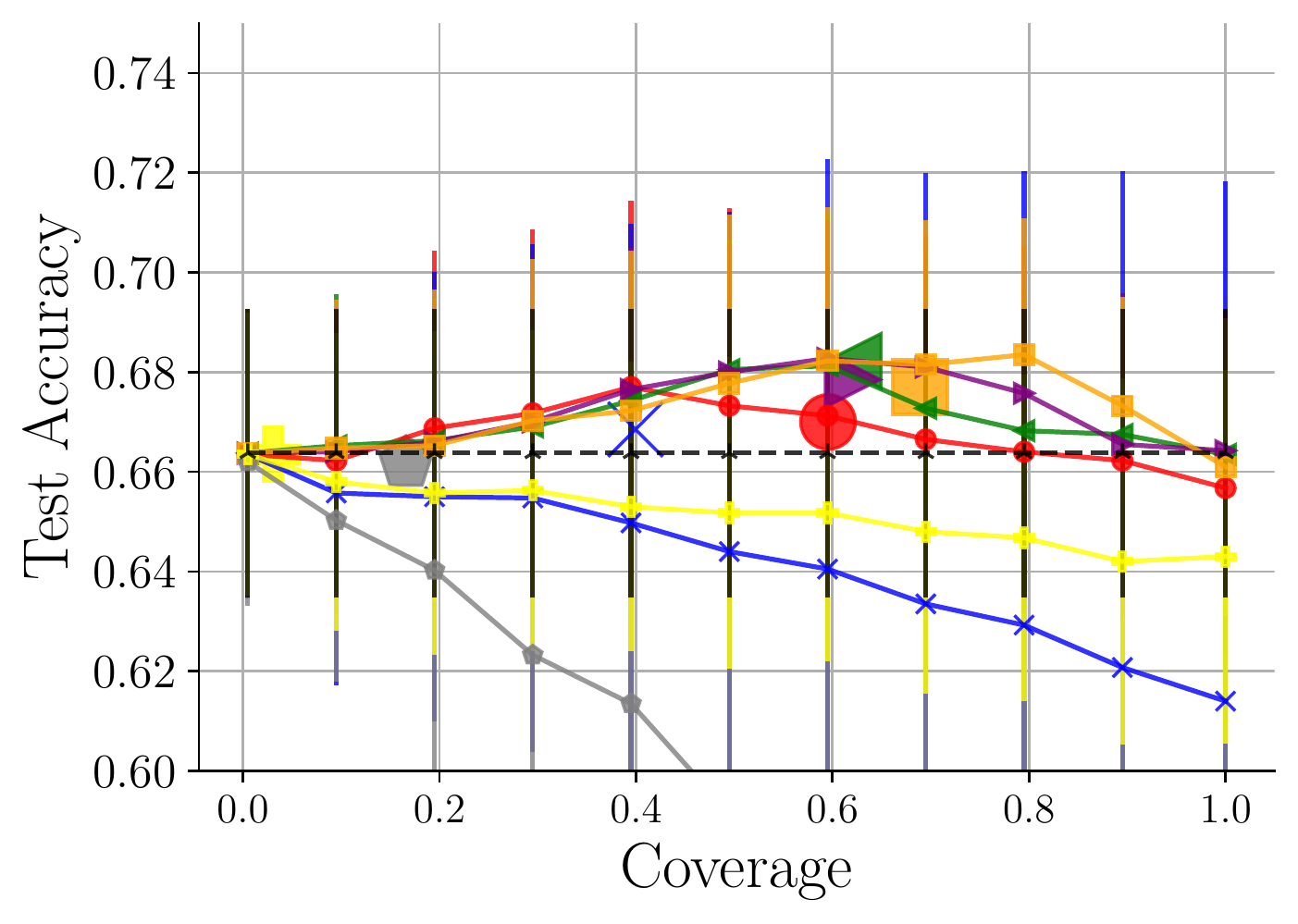}
        \subcaption{COMPASS}
    \end{subfigure} 
    \begin{subfigure}{0.5\textwidth}
        \centering
          \includegraphics[width=\textwidth]{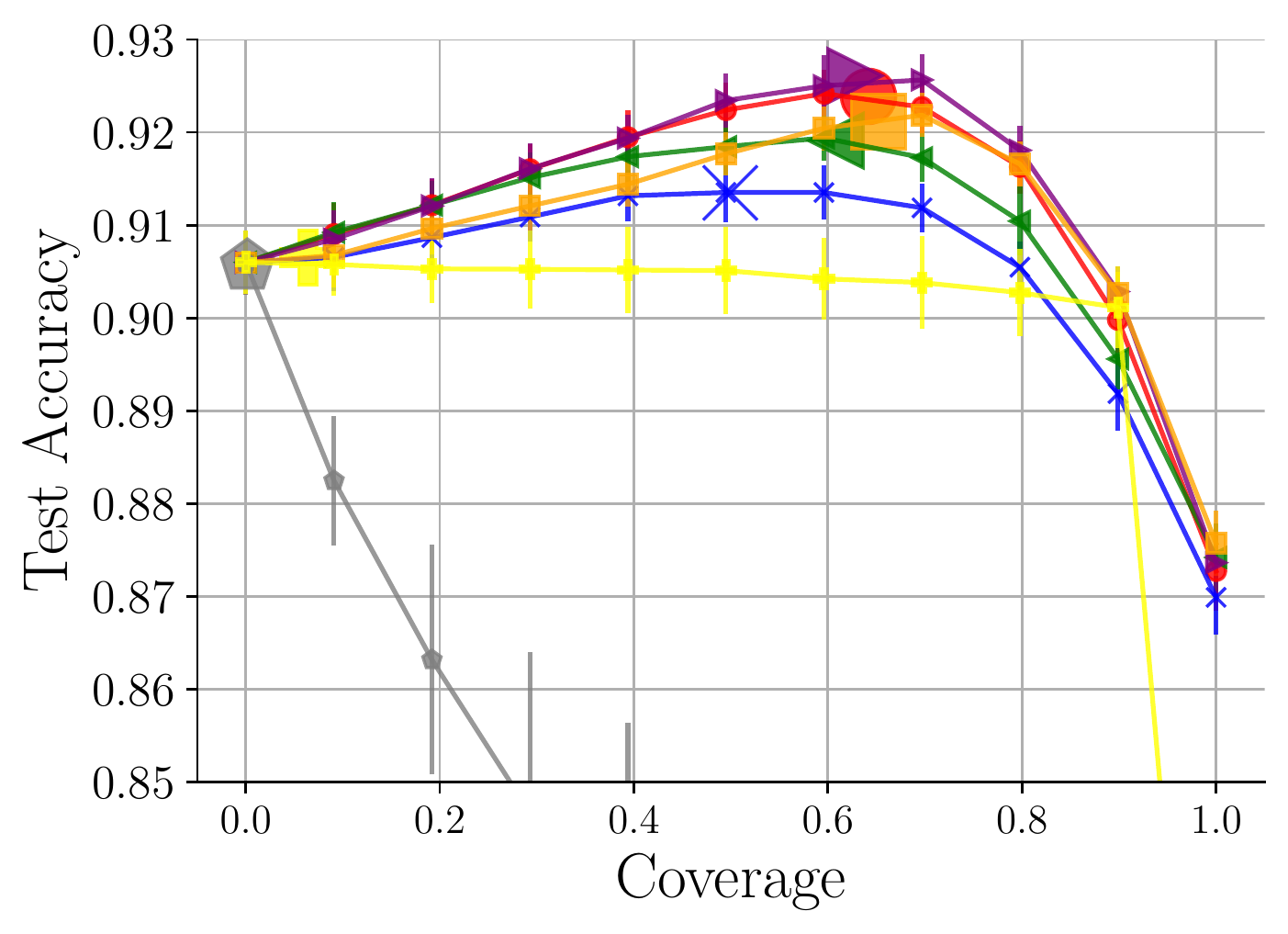}
        \subcaption{HateSpeech}
    \end{subfigure} 
            \begin{subfigure}{0.5\textwidth}
        \centering
          \includegraphics[width=\textwidth]{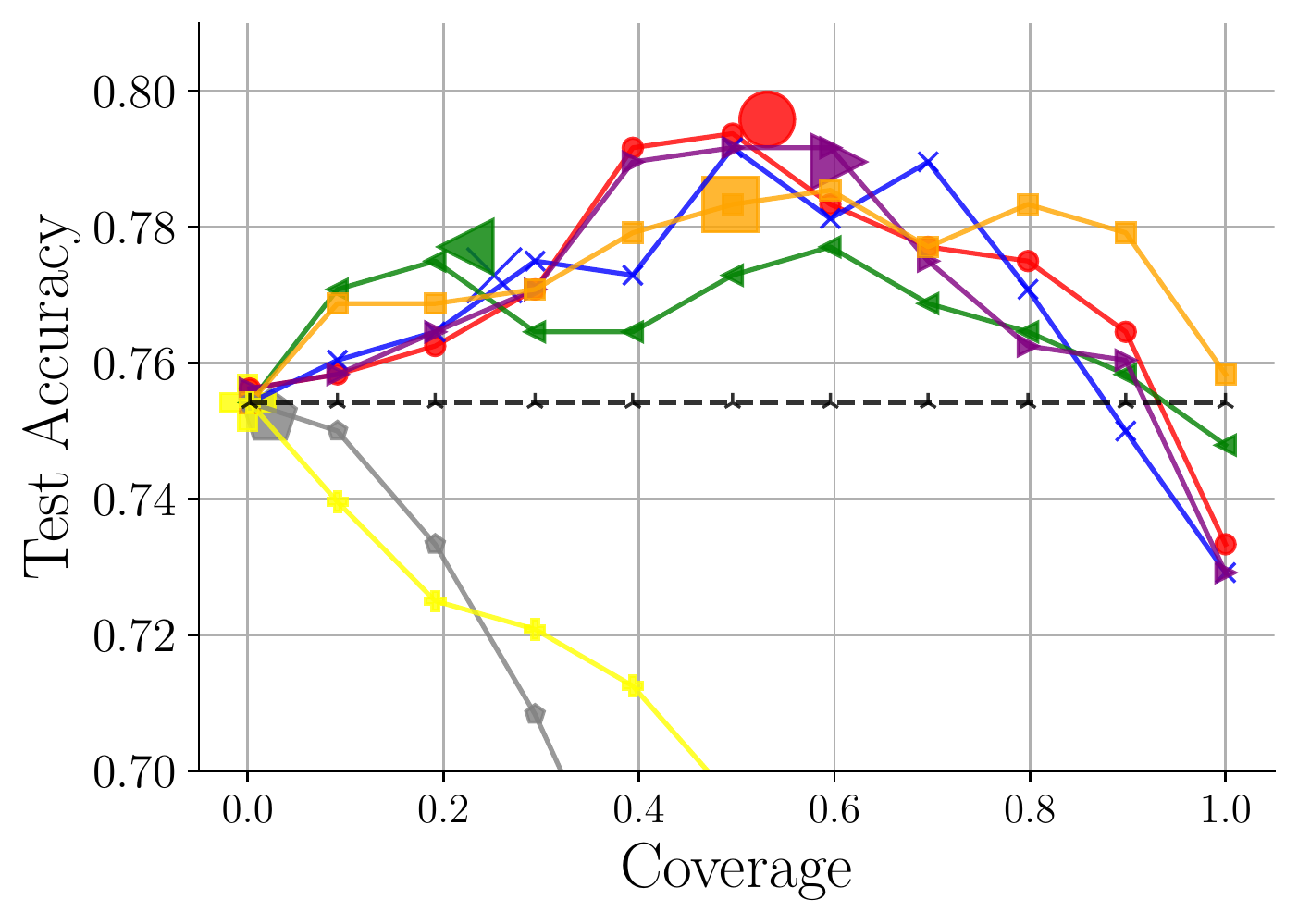}
        \subcaption{ImageNet-16H}
    \end{subfigure}
              \begin{subfigure}{0.5\textwidth}
        \centering
          \includegraphics[width=\textwidth]{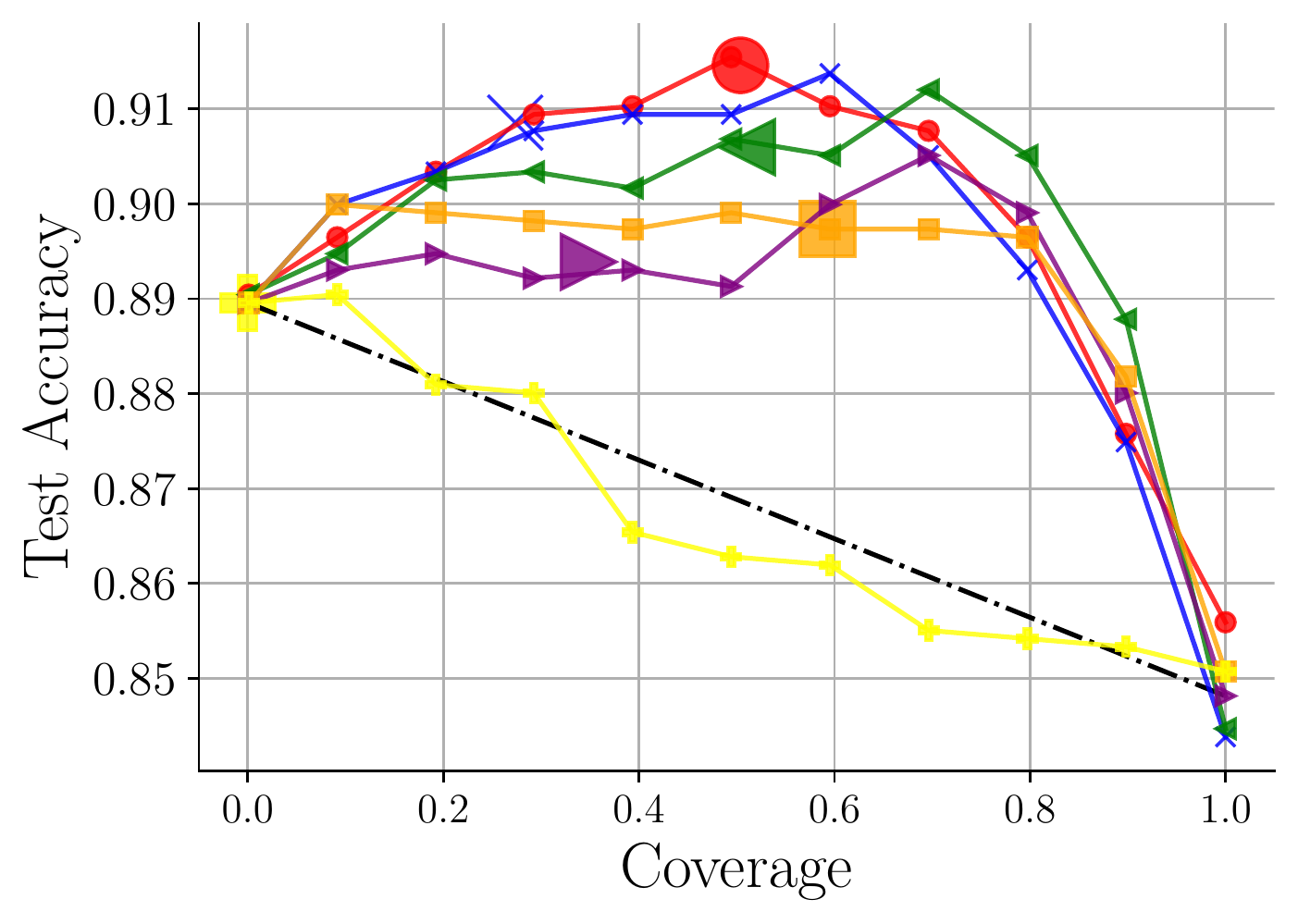}
        \subcaption{ Chest X-ray - Airspace Opacity}
    \end{subfigure}
            \begin{subfigure}{0.5\textwidth}
        \centering
          \includegraphics[width=\textwidth]{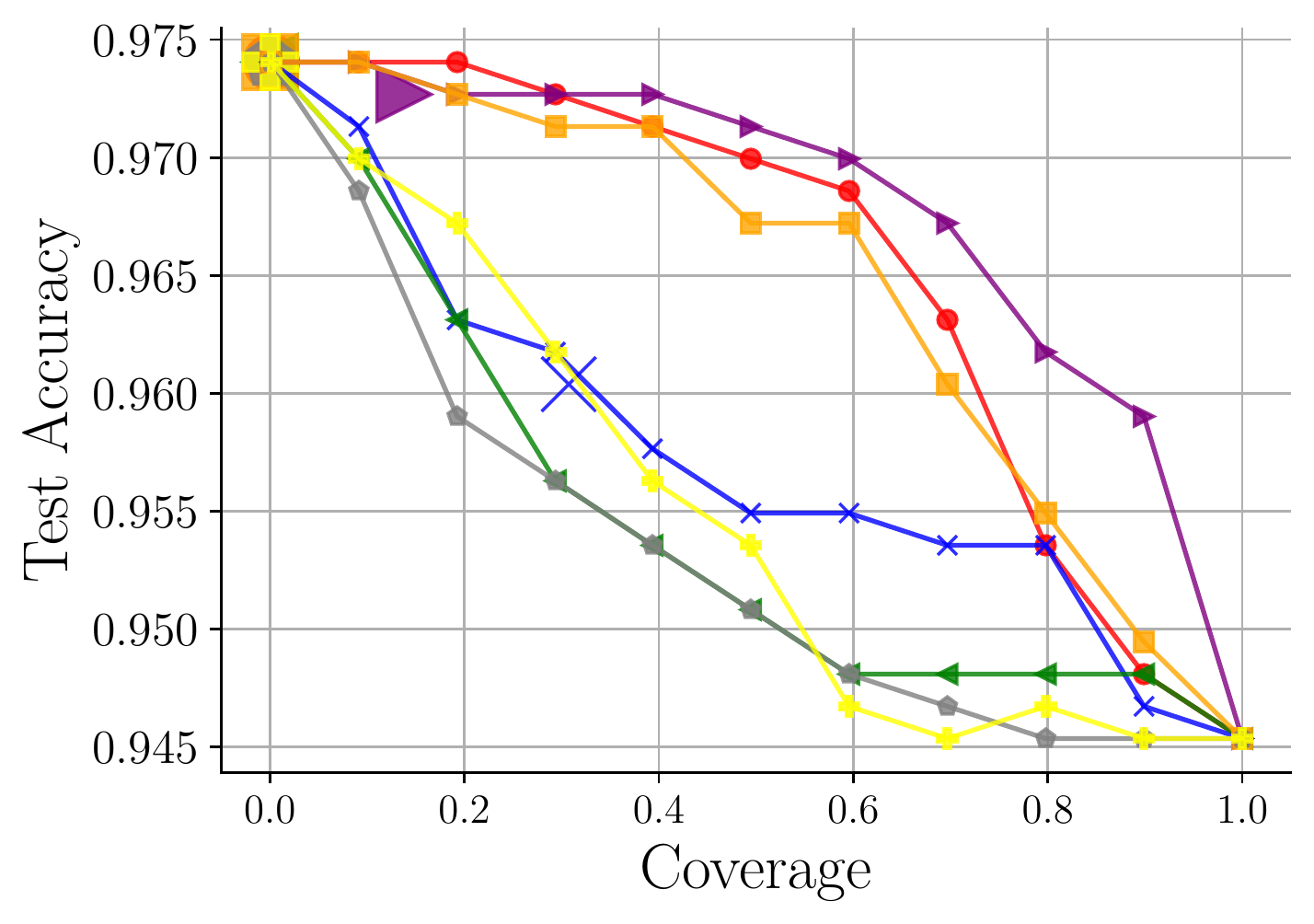}
        \subcaption{ Chest X-ray - Pneumothorax}
    \end{subfigure}
    \caption{\looseness=-1 Accuracy vs coverage (fraction of points where classifier predicts) plots across the real world datasets showcasing the behavior of our method and the baselines. On each plot, we showcase the test accuracy of each method with a large marker, with the curve representing varying the rejector threshold on the test set.  To achieve different levels of coverage, we sort the rejection score for each method on the test set and vary the threshold used, for \realizablesurrogate the rejector is defined as $r(x)=\bI_{g_{\bot}(x) - \max_y g_y(x) \geq c}$ where the optimal solution is at $c=0$ and we vary $c \in \mathbb{R}$ to obtain the curve. }
      \label{fig:real_datasets}
\end{figure}
\newpage

\begin{figure}[H]
    \centering
    \begin{subfigure}{0.5\textwidth}
        \centering
          \includegraphics[width=\textwidth]{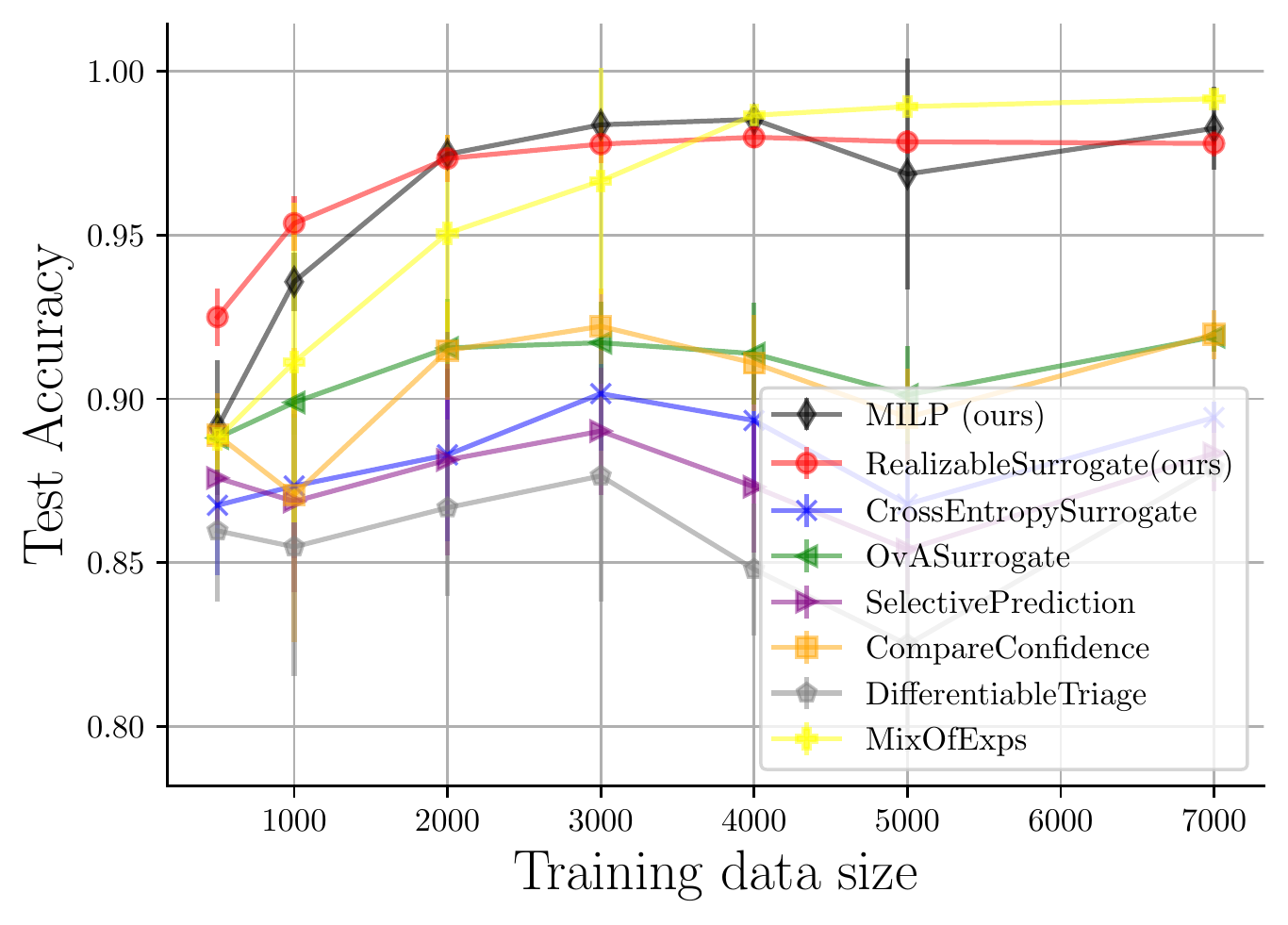}
        \subcaption{Synthetic Data Sample Complexity}
        \label{fig:synthetic-res}%
    \end{subfigure}\hfill%
    \begin{subfigure}{0.5\textwidth}
        \centering
          \includegraphics[width=\textwidth]{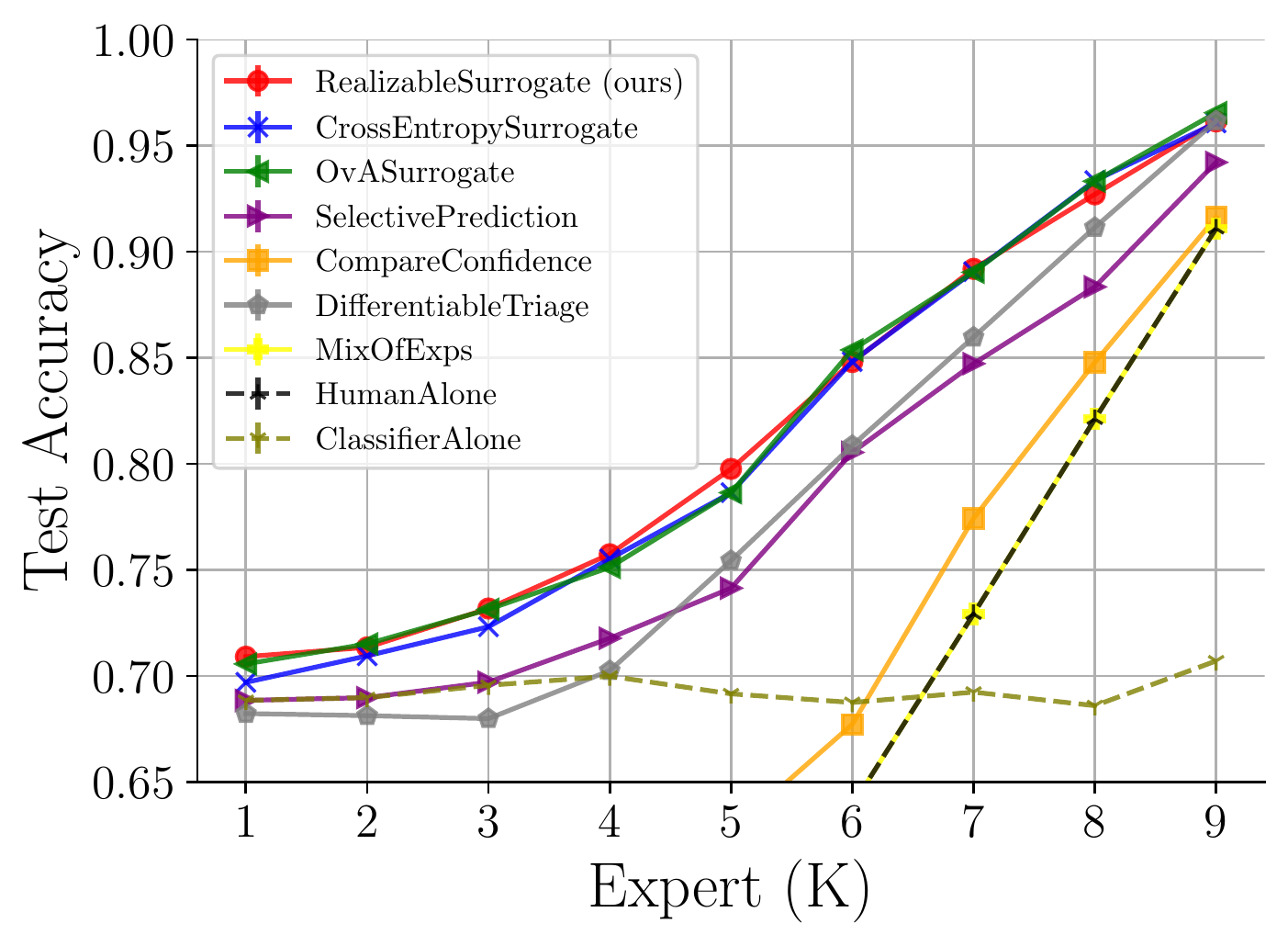}
        \subcaption{CIFAR-K Semi-Synthetic}
        \label{fig:cifark-res}%
    \end{subfigure}
    \caption{(a) Test performance of the different methods on synthetic data as we increase the training data size and repeat the randomization over 10 trials to get standard errors. (b) Test performance on the semi-synthetic CIFAR-K dataset vs.~the number of classes $K$ for which the expert is perfect.}
    \vspace{-6mm}
      \label{fig:semi-and-synthetic_results}
\end{figure}

\begin{table*}[t]
    \centering
        \caption{Datasets used for our benchmark for learning with deferral to humans. We note the total number of samples $n$, the target set size $|\mathcal{Y}|$,  the number of tasks in each dataset (a task is a set of human and target labels), the human expert where 'random annotator' means that for each point we have multiple human annotations and we let the target be a consensus and the human label be a random sample while 'separate human annotation' means that the human label is completely separate from the label annotations and finally the model class for both the classifier and rejector.}
    \resizebox{1\textwidth}{!}{
    \begin{tabular}{p{0.3 \textwidth}ccllp{0.4 \textwidth}}
    \toprule
         \textbf{Dataset} & $n$ & $\left| \mathcal{Y} \right|$ & \textbf{ Number of Tasks}  & \textbf{Human} & \textbf{Model Class} \\ 
         \toprule
         SyntheticData (ours) & arbitrary & 2 & 1 & synthetic &linear\\
         CIFAR-K  & 60k& 10 & 10 (per expert $k$) & synthetic (perfect on k classes) & CNN\\
         \midrule
         CIFAR-10H \citep{battleday2020capturing} & 10k & 10 & 1 & separate human annotation & pretrained WideResNet \citep{zagoruyko2016wide} \\
         Imagenet-16H \citep{kerrigan2021combining} & 1.2k & 16 & 4 (per noise version) & separate human annotation & pretrained  DenseNet121 \citep{huang2017densely}, finetuning last layer only \\
         HateSpeech \citep{davidson2017automated} & 25k & 3 & 1 & random annotator & FNN on  embeddings from SBERT \citep{reimers2019sentence} \\
         COMPASS \citep{dressel2018accuracy} & 1k & 2 & 1 & separate human annotation &  linear\\
         NIH Chest X-ray  \citep{wang2017hospital,majkowska2020chest}
 & 4k & 2 & 4 (for different conditions)  & random annotator& pretrained DenseNet121 on non-human labeled data \\ \bottomrule
    \end{tabular}}
    \label{tab:datasets}
\end{table*}

\subsection{Human-AI Deferral Benchmark}
\textbf{Objective.} We investigate the empirical performance of our proposed approaches compared to prior baselines on a range of datasets. Specifically, we want to compare the accuracy of the human-AI team at the learned classifier-rejector pairs. We also check the accuracy of the system when we change the deferral policy by varying the threshold used for the rejector, this leads to an accuracy-coverage plot where \emph{coverage} is defined as the fraction of the test points where the classifier predicts. 

\textbf{Datasets.} In Table \ref{tab:datasets} we list the datasets used in our benchmark. We start with synthetic data described below, then semi-synthetic data with CIFAR-K \citep{mozannar2020consistent}. We then evaluate on 5 real world datasets with three image classification domains with multiple tasks per domain, a natural language domain and a tabular domain. Each dataset is randomly split 70-10-20 for training-validation-testing respectively.

\textbf{Baselines.} We compare to multiple methods from the literature including: the confidence method from \citet{raghu2019algorithmic} (CompareConfidence), the surrogate $L_{CE}^\alpha$ from \citet{mozannar2020consistent} (CrossEntropySurrogate), the surrogate $\Psi_{\textrm{OvA}}$ from \citet{verma2022calibrated} (OvASurrogate),  Diff-Triage from \citet{okati2021differentiable} (DifferentiableTriage),
mixture of experts from \cite{madras2018predict} (MixOfExps)
and finally a selective prediction baseline that thresholds classifier confidence for the rejector (SelectivePrediction). 
For all baselines and datasets, we train using Adam and use the same learning rate and the same number of training epochs to ensure an equal footing across baselines, each run is repeated for 5 trials with different dataset splits. We track the best model in terms of system accuracy on a validation set for each training epoch and return the best-performing model. For \realizablesurrogate, we perform a hyperparameter search on the validation set over $\alpha \in [0,1]$, and do hyperparameter tuning over $L_{CE}^{\alpha}$.

\subsection{Synthetic and Semi-Synthetic Data}\label{sec:exp_synthetic}

\textbf{Synthetic Data.} We create a set of synthetic data distributions that are realizable by linear functions (or nearly so) to benchmark our approach. For the input $X$, we set the dimension $d$, and experiment with two data distributions. (1) Uniform distribution: we draw points $X \sim \mathrm{Unif}(0,U)^d$ where $U \in \mathbb{R}^+$; (2) Mixture-of-Gaussians: we fix some $K \in \mathbb{N}$ and generate data from $K$ equally weighted Gaussians, each with random uniform means and variances.
To obtain labels $Y$ that satisfy Assumption \ref{ass:realizability}, we generate two random halfspaces and denote one as the optimal classifier $m^*(x)$ and the other as the optimal rejector $r^*(x)$. 
We then set the labels $Y$ on the side where $r^*(x)=0$ to be consistent with $m^*(x)$ with probability $1-p_m$ and otherwise uniform.
When $r^*(x)=1$, we sample the labels uniformly.
Finally, we choose the human expert to have error $p_{h0} $ when $r^*(x)=0$ and have error $p_{h1}$ when $r^*(x)=1$. 
When $p_m=0, p_{h0} \in [0,1]$, and $p_{h1}=0$, this process generates datasets $D = \{x_i,y_i,h_i\}_{i=1}^n$ that satisfy Assumption \ref{ass:realizability}.
\begin{figure}[h]
    \centering
    \includegraphics[scale=0.7]{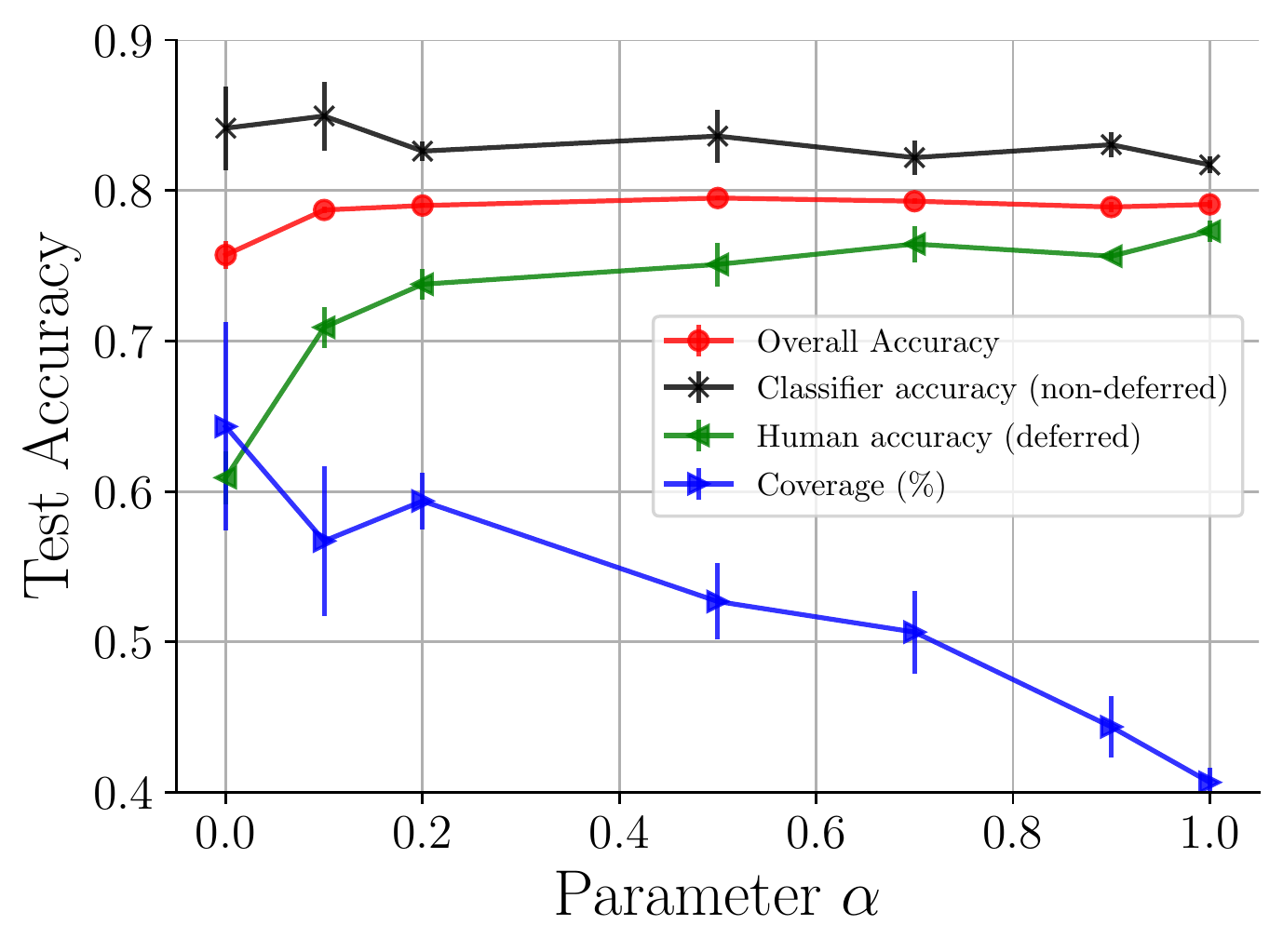}
    \caption{Sensitivity of the \realizablesurrogate to the hyperparameter $\alpha$. We vary the hyperparameter $\alpha$ in the \realizablesurrogate surrogate loss  and show the different metrics including overall accuracy, accuracy when we defer, accuracy when we don't defer, and finally coverage.}
    \label{fig:effect_alpha}
\end{figure}

\textbf{Sample Complexity.} \looseness=-1 For realizable data with a feature distribution that is mixture of Gaussians ($d=30$, $p_m=0, p_{h0}=0.3, p_{h1}=0$), Figure \ref{fig:synthetic-res} plots the test accuracy of the different methods on a held-out dataset of 5k points as we increase the training data size.
We observe that MILP and \realizablesurrogate are able to get close to zero error, while all other methods fail at finding a near zero-error solution.
We also experiment with non-realizable data. 
For example, when $p_m=0.1, p_{h0} =0.4, p_{h1}=0.1$ with $n=1000$, the optimal test error is $7.5 \pm 1.0 \%$  for the generated data: the MILP obtains $11.2$ error and \realizablesurrogate achieves $17.8 \pm 1.0$ error, while the best baseline CrossEntropySurrogate achieves $21.4 \pm 1.1$ error.
In the Appendix, we show results on the uniform data distribution, which shows an identical pattern, and we study the run-time and performance of the MILP as we increase the error probabilities.  

\textbf{CIFAR-K.} We use the CIFAR-10 image
classification dataset \citep{krizhevsky2009learning} and employ a simple  convolution neural network (CNN) with three layers. 
 We consider the human expert models from \citet{mozannar2020consistent,verma2022calibrated}:  if the image is in the first $K$ classes the expert is perfect, otherwise the expert predicts randomly.  Figure \ref{fig:cifark-res} shows the test accuracy of the different methods as we vary the expert strength $K$. \realizablesurrogate outperforms the second-best method by 0.8\% on average and up to 2.8\% maximum showcasing that the method can perform well for non-linear predictors. 

\subsection{Realistic Data}

\textbf{Models.} \looseness=-1 In Figure \ref{fig:real_datasets}, we showcase the test accuracy of the different baselines on the real datasets in Table \ref{tab:datasets}, and illustrate their behavior when we constrain our method and the baselines to achieve different levels of coverage. The test accuracy of the operating point on the different datasets is shown in Table \ref{tab:results_summarized}.
We can see that $L_{RS}^{\alpha}$ is competitive with the best baseline on each dataset/task.
Moreover, we see that the human-AI team is often able to achieve performance that is higher than the human or classifier on their own. The methods often achieve peak performance at a coverage rate that is not at the extremes of [0,1], and on each of the six datasets we notice variability between the peak accuracy coverage rate  indicating tat they are finding different solutions. 
This demonstrates that deferral using $L_{RS}^{\alpha}$ is able to achieve complementary human-AI team performance in practice.
In summary, the new surrogate $L_{RS}$ performs as well as the MILP on synthetic data, and as well as all the baselines (or better) on real-world data. Note that Differentiable Triage on these datasets is underperforming as we are testing it on a setting beyond the paper as here we only have samples of expert predictions instead of probabilities from the expert.

\begin{table}[h!]
\centering
\caption{Test accuracy of the operating point of the different methods on the datasets tested on. The baselines are $L_{CE}$ \citep{mozannar2020consistent}, $\Psi_{\textrm{OvA}}$ \citep{verma2022calibrated}, Selective Prediction (SP), Compare Confidence (CP) \citep{raghu2019algorithmic}, DIFT \citep{okati2021differentiable}  and MoE \citep{madras2018predict}.}

\resizebox{\linewidth}{!}{
\begin{tabular}{cccccccc}
\hline
\textbf{Dataset} & $L_{RS}^{\alpha}$ (ours) & $L_{CE}$ & $\Psi_{\textrm{OvA}}$ & SP & CC & DIFT & MoE \\
\hline
Synthetic Realizable& 0.979 & 0.891 & 0.918 & 0.882 & 0.918 & 0.870 & \textbf{0.992} \\
Synthetic Non-Realizable &\textbf{ 0.879 }& 0.828 & 0.839 & 0.797 & 0.836 & 0.770 & 0.774 \\
Cifar-K (K=5) & \textbf{0.795} & 0.785 & 0.786 & 0.747 & 0.621 & 0.749 & 0.550 \\
Compass & 0.670 & 0.668 & \textbf{0.682} & 0.678 & 0.677 & 0.662 & 0.663 \\
Cifar-10H & \textbf{0.969} & 0.960 & 0.963 & 0.966 & 0.968 & 0.949 & 0.953 \\
Hate Speech & 0.924 & 0.913 & 0.919 & \textbf{0.926} & 0.921 & 0.906 & 0.907 \\
ImageNet16H (noise 80)& \textbf{0.912} & 0.908 & 0.909 & 0.910 & 0.908 & 0.898 & 0.904 \\
ImageNet16H (noise 95)& 0.865 & 0.872 & 0.872 & \textbf{0.875} & 0.868 & 0.856 & 0.861 \\
ImageNet16H (noise 110)& 0.802 & 0.791 & 0.791 & \textbf{0.809} & 0.792 & 0.756 & 0.761 \\
ImageNet16H (noise 125)& 0.755 & 0.707 & 0.732 & \textbf{0.756} & 0.743 & 0.655 & 0.604 \\
Pneumothorax & 0.976 & 0.963 & \textbf{0.978} & 0.972 & \textbf{0.978} & \textbf{0.978} & \textbf{0.978} \\
Airspace Opacity & \textbf{0.913} & 0.908 & 0.906 & 0.899 & 0.905 & 0.894 & 0.894 \\
\hline \\
\end{tabular}}
\label{tab:results_summarized}
\end{table}

\paragraph{Hyperparameter $\alpha$.} We show how the behavior of the classifier and rejector system changes when we modify the hyperparameter $\alpha \in [0,1]$ in Figure \ref{fig:effect_alpha}. When $\alpha$ is small, the behavior of the surrogate is the same as selective prediction which is why we see the lowest accuracy of the human when we defer. As $\alpha$ increases to $1$, we can see that the system better adapts to the human. 

\paragraph{Recommendations: Which Method to Use?} Given our experimental results, the question to ask is which method should be used for a given dataset and model class. The simple and natural baseline of CompareConfidence should be the first tool one applies to their setting, it often achieves good performance,  outperforming the naive baseline SelectivePrediction. However, CompareConfidence does not allow the classifier to adapt to the humans strengths and weaknesses. The surrogates CrossEntropySurrogate and OvASurrogate when applied with expressive model classes such as deep networks can find complementary classifiers. The surrogates offer other advantages, notably, CrossEntropySurrogate has been shown to have better sample complexity over the CompareConfidence baseline and can incorporate arbitrary costs of deferral and prediction \citep{mozannar2020consistent}.
However, as our synthetic experiments have shown, there is a limit of the CrossEntropySurrogate and OvASurrogate surrogats to how much they can complement the human and defer accordingly. This is where our proposed methods MILPDefer and \texttt{RealizableSurrogate} come in. We recommend using the MILP in settings with limited data where linear models are suitable as it can achieve optimal performance, however, one must carefully tune regularization parameters to not overfit. If the data is realizable, then the \texttt{RealizableSurrogate} is also optimal and is much easier to optimize, one can apply the surrogate without knowing beforehand if the data is realizable. \texttt{RealizableSurrogate} works well with linear and non-linear model classes, and performs the best under model resource constraints, we recommend using it broadly when optimizing accuracy.

\section{Conclusion}
\looseness=-1
We have shown that properly learning halfspaces with deferral (LWD-H) is computationally hard and that existing approaches in the literature fail in this setting.
Understanding the computational limits of learning to defer led to the design of a new exact algorithm (the MILP) and a new surrogate (\realizablesurrogatenosp) that both obtain better empirical performance than existing surrogate approaches.
Studying $(\cM,\cR)$-consistency in the non-realizable setting, obtaining conditions under which nonconvex surrogates like $L_{RS}$ can be provably and efficiently minimized, and considering \emph{online} versions of learning to defer are interesting directions for future work.
As human-AI teams are deployed in real-world decision-making scenarios, better and safer methods for training these systems are of critical interest. Giving the AI the power to allow the human to predict or not  requires very careful optimization of the rejector so that we have favorable outcomes, this motivates the need for exact algorithms with guarantees.

\section*{Acknowledgments}
HM is thankful for the support of the MIT-IBM Watson AI Lab.  

\bibliography{ref}
\appendix
\onecolumn

\section{Practitioner's guide to our approach}\label{apx:guide}

\subsection{MILP}

We implement the MILP \eqref{eq:milp_obj}-\eqref{eq:mil_last_cst} in the binary setting using the Gurobi Optimizer \cite{gurobi} in Python.  

\begin{python}
class MILPDefer:
    def __init__(self, n_classes, time_limit=-1, add_regularization=False,
                lambda_reg=1, verbose=False):
        self.n_classes = n_classes
        self.time_limit = time_limit
        self.verbose = verbose
        self.add_regularization = add_regularization
        self.lambda_reg = lambda_reg

    def fit(self, dataloader_train, dataloader_val, dataloader_test):
        self.fit_binary(dataloader_train, dataloader_val, dataloader_test)

    def fit_binary(self, dataloader_train, dataloader_val, dataloader_test):
        data_x = dataloader_train.dataset.tensors[0]
        data_y = dataloader_train.dataset.tensors[1]
        human_predictions = dataloader_train.dataset.tensors[2]

        C = 1
        gamma = 0.00001
        Mi = C + gamma
        Ki = C + gamma
        max_data = len(data_x)
        hum_preds = 2*np.array(human_predictions) - 1
        # add extra dimension to x
        data_x_original = torch.clone(data_x)
        norm_scale = max(torch.norm(data_x_original, p=1, dim=1))
        last_time = time.time()
        # normalize data_x and then add dimension
        data_x = torch.cat((torch.ones((len(data_x)), 1),
                           data_x/norm_scale), dim=1).numpy()
        data_y = 2*data_y - 1  # covert to 1, -1
        max_data = max_data  # len(data_x)
        dimension = data_x.shape[1]

        model = gp.Model("milp_deferral")
        model.Params.IntFeasTol = 1e-9
        model.Params.MIPFocus = 0
        if self.time_limit != -1:
            model.Params.TimeLimit = self.time_limit

        H = model.addVars(dimension, lb=[-C] *
                          dimension, ub=[C]*dimension, name="H")
        Hnorm = model.addVars(
            dimension, lb=[0]*dimension, ub=[C]*dimension, name="Hnorm")
        Rnorm = model.addVars(
            dimension, lb=[0]*dimension, ub=[C]*dimension, name="Rnorm")
        R = model.addVars(dimension, lb=[-C] *
                          dimension, ub=[C]*dimension, name="R")
        phii = model.addVars(max_data, vtype=gp.GRB.CONTINUOUS, lb=0)
        psii = model.addVars(max_data, vtype=gp.GRB.BINARY)
        ri = model.addVars(max_data, vtype=gp.GRB.BINARY)

        equal = np.array(data_y) == hum_preds * 1.0
        human_err = 1-equal

        if self.add_regularization:
            model.setObjective(gp.quicksum([phii[i] + ri[i]*human_err[i]
            for i in range(max_data)])/max_data + self.lambda_reg * gp.quicksum(
                [Hnorm[j] for j in range(dimension)])
                + self.lambda_reg * gp.quicksum([Rnorm[j] for j in range(dimension)]))
        else:
            model.setObjective(gp.quicksum(
                [phii[i] + ri[i]*human_err[i] for i in range(max_data)])/max_data)
        for i in range(max_data):
            model.addConstr(phii[i] >= psii[i] - ri[i], name="phii" + str(i))
            model.addConstr(Mi*psii[i] >= gamma - data_y[i]*gp.quicksum(
                H[j] * data_x[i][j] for j in range(dimension)), name="psii" + str(i))
            model.addConstr(gp.quicksum([R[j]*data_x[i][j] for j in range(dimension)]) >=
            Ki*( ri[i]-1) + gamma*ri[i], name="Riub" + str(i))
            model.addConstr(gp.quicksum([R[j]*data_x[i][j] for j in range(
                dimension)]) <= Ki*ri[i] + gamma*(ri[i]-1),  name="Rilb" + str(i))
            model.update()
        if self.add_regularization:
            for j in range(dimension):
                model.addConstr(Hnorm[j] >= H[j], name="Hnorm1" + str(j))
                model.addConstr(Hnorm[j] >= -H[j], name="Hnorm2" + str(j))
                model.addConstr(Rnorm[j] >= R[j], name="Rnorm1" + str(j))
                model.addConstr(Rnorm[j] >= -R[j], name="Rnorm2" + str(j))

        model.ModelSense = 1  # minimize
        model._time = time.time()
        model._time0 = time.time()
        model._cur_obj = float('inf')
        # model.write('model.lp')
        if self.verbose:
            model.optimize()
        else:
            model.optimize()
        # check if halspace solution has 0 error
        error_v = 0
        rejs = 0
        for i in range(max_data):
            rej_raw = np.sum([R[j].X * data_x[i][j] for j in range(dimension)])
            pred_raw = np.sum([H[j].X * data_x[i][j]
                              for j in range(dimension)])
            if rej_raw > 0:
                rejs += 1
                error_v += (data_y[i] * hum_preds[i] != 1)
            else:
                pred = (pred_raw > 0)
                error_v += (data_y[i] != (2*pred-1))

        self.H = [H[j].X for j in range(dimension)]
        self.R = [R[j].X for j in range(dimension)]
        self.run_time = model.Runtime
        self.norm_scale = norm_scale
        self.train_error = error_v/max_data
\end{python}

\subsection{Realizable Surrogate}

We implement the \realizablesurrogate in PyTorch. We showcase the loss function $L_{RS}$ below:
\begin{python}
def realizable_surrogate_loss(outputs, human_is_correct, labels, lambdaa):
    '''
    outputs (tensor): outputs of model with K+1 output heads (without softmax)
    human_is_correct (tensor): binary tensor indicating if human is
                    correct on each point I_{h=y}
    labels (tensor): list of targets y_i
    lambdaa (float in [0,1]): trade-off parameter in loss
    
    return: loss (single tensor)
    '''
    batch_size = outputs.size()[0]            
    outputs_exp = torch.exp(outputs)
    rs_loss =  -torch.log2(( m * outputs_exp[range(batch_size), -1] 
    + outputs_exp[range(batch_size),labels] )  /(torch.sum(outputs_exp, dim = 1) +eps_cst ))   
    ce_loss = -torch.log2(( outputs_exp[range(batch_size),labels] ) 
            /(torch.sum(outputs_exp[range(batch_size),:-1], dim = 1)  +eps_cst ))  
    loss = lambdaa*rs_loss + (1-lambdaa)*ce_loss 
    return torch.sum(loss)/batch_size

\end{python}


\section{MILP}\label{apx:milp}

\subsection{Verification}\label{apx:milp_verify}
The MILP in the binary setting is formulated as:
\begin{align}
M^*, R^*, . &= \arg \min_{M,R,\{r_i\},\{t_i\},\{\phi_i\}} \sum_{i} \phi_i + r_i \bI_{h_i \neq y_i} \\
& \phi_i \geq t_i - r_i, \qquad \phi_i \geq 0 \quad  \forall i \in [n] \\
& K_m t_i \geq \gamma_h - y_i M^\top x_i \quad \forall i \in [n] \label{eq: constraint_milp_ti} \\
& R^{\top} x_i \leq K_r r_i + \gamma_r (r_i - 1) , \quad  R^{\top} x_i \geq K_r (r_i - 1) + \gamma_r r_i  \quad \forall i \in [n]\\ 
& -C \leq R_i \leq C , \quad -C \leq M_i \leq C \quad \forall i \in [d] \\
& r_i \in \{0,1\}, t_i \in \{0,1\}, \phi_i \in \mathbb{R}^+ \quad \forall i \in [n],
\ R,M \in \mathbb{R}^d
\end{align}

\textbf{Extension to Multiclass.} The above MILP only applies to binary labels but we can generalize it to the multiclass setting where $\mathcal{Y} = \{1,\cdots, C\}$. In this case, we have a coefficient vector $M_j$ for each class $j \in \mathcal{Y}$, and $m(x) = \arg \max_{j \in \mathcal{Y}} M_j^\top x$. Given a labeled point $(x,y)$, we let $c_j = \textrm{sign}( M_y^\top x - M_j^\top x) $ for $j \neq y$, and let $t_i = \bI_{\sum_{j \neq y} c_j < C-1}$. Then if $m(x) = y$, we must have $c_j = 1$ for all $j \neq y$ and thus $t_i =0$ which means that the classifier is correct. Similarly, if there exists a $j \neq y$ for which $c_j = -1$, it means the classifier is incorrect and accordingly $t_i =1$. We can reformulate these indicator constraints using a similar big-M approach as above. The formulation is below:

\begin{align}
M^*, R^*, . &= \arg \min_{M,R,\{r_i\},\{t_i\}, \{c_{ij}\},\{\phi_i\}} \sum_{i} \phi_i + r_i \bI_{h_i \neq y_i} \\
& \phi_i \geq t_i - r_i, \qquad \phi_i \geq 0 \quad  \forall i \in [n] \\
& (M_{y_i} - M_{j})^\top x_i \leq 2K_h c_{ij}  + \gamma_h (c_{ij} - 1) , \nonumber
\\&(M_{y_i} - M_{j})^\top x_i \geq 2K_h (c_{ij}  - 1) + \gamma_h c_{ij}   \quad \forall i \in [n] \ \forall j \in [C] \neq y_i \\
& t_i \geq (C - 1 - \sum_{j \in [L], j != y_i} c_{ij})/(C-1) \\
& R^{\top} x_i \leq K_r r_i + \gamma_r (r_i - 1) , \quad  R^{\top} x_i \geq K_r (r_i - 1) + \gamma_r r_i  \quad \forall i \in [n]\\ 
& -C \leq R_i \leq C , \quad -C \leq M[i,l] \leq C \quad \forall i \in [d] \ \forall l \in [C] \\
& r_i \in \{0,1\}, t_i \in \{0,1\}, c_{ij} \in \{0,1\}, \phi_i \in \mathbb{R}^+ \quad \forall i \in [n],
\ R,M \in \mathbb{R}^d
\end{align}

Let us verify the formulations above.

The variable $\phi_i \geq \max(t_i- r_i,0)$, the RHS takes values either 0 or 1, since $\phi_i$ in the objective then the optimal value is either $0$ or $1$ as well so that $\phi_i = \max(t_i - r_i,0) = (1-r_i)t_i$. 

For $t_i$ in the binary case:  when $y_i M^\top x_i$ is positive, then $\gamma_h - y_i M^\top x_i$ is negative since $|M^\top x_i| \geq \gamma_h$ by Assumption \ref{ass: margin}, so that to satisfy constraint \eqref{eq: constraint_milp_ti} either value of $0$ or $1$ are valid for $t_i$, however since $t_i$ shows up in the objective then the optimal value is $0$. On the other hand, when $y_i M^\top x_i$ is negative, then $\gamma_h - y_i M^\top x_i$ is positive, so that the only valid option for $t_i$ is $1$ and since $M^\top x_i \leq K_m$ then the constraint can be satisfied. So that we proved that $t_i = sign(y_i M^\top x_i)$.

We previously verified constraint for $r_i$ and $R$ in the body. When $r_i=0$ then we have the constraints $ R^{\top} x_i  \leq - \gamma_r$ and $R^{\top} x_i \geq -K_r$: this forces the rejector to be negative which is consistent. When $r_i=1$, we have $R^{\top} x_i  \geq  \gamma_r$ and $R^{\top} x_i \leq K_r$: which means the rejector is positive.  Thus we proved $r_i = \bI(R^\top x_i \geq 0 )$.

For $t_i$ in the multiclass settings: by analogy to the constraints for $R$ and $r_i$ it is easy to see that the variable $c_{ij} = sign( H_{y_i}^\top x_i - H_{j}^\top x_i)$. For a given $x_{i},y_{i}$, the classification is only correct if $c_{ij}=1$ for all $j \in [C] \neq y_i$ so that $\arg\max_j H_i^\top x_i =y_i$. We can then see that we set $t_i = \bI( \sum_{j \neq y_i} c_{ij} /(C-1) \neq 1 )$ so that $t_i$ denotes the error of our classifier on example $i$.

\section{Experimental Details and Results}\label{apx:experiments}

\subsection{Baseline Implementation}\label{apx:baselines}

OvASurrogate \citep{verma2022calibrated}: We rely on the loss implementation available online at \footnote{\url{https://github.com/rajevv/OvA-L2D}}.

DifferentiableTriage \citep{okati2021differentiable}: We rely on the implementation in \footnote{\url{https://github.com/Networks-Learning/differentiable-learning-under-triage}}. Note that the differentiable triage method implementation in \cite{okati2021differentiable} relies on having loss estimates of the human, particularly cross entropy loss estimates, which requires the conditional probabilities $\bP(H=i|X=x)$ for each $i \in \mathcal{Y}$. However, in our setting, we only have samples of the human decisions $m_i$, not probabilistic estimates. The method can be summarized as a two-stage method: 1) classifier training: at each epoch, only train on points where classifier loss is lower than human loss, 2) rejector training: fit the rejector to predict who between the classifier and the human has lower loss. Since we only have samples of human behavior, we use the $0-1$ loss of the classifier and the human on an example basis for comparison. 

CrossEntropySurrogate \citep{mozannar2020consistent}: We rely on the implementation in \footnote{\url{https://github.com/clinicalml/learn-to-defer}}. We tune the parameter $\alpha$ over the grid $[0, 0.1, 0.5, 1]$ on the validation set.

CompareConfidence \citep{raghu2019algorithmic}: we train the classifier using the cross entropy loss on all the data, we then train a model to predict if the human is correct or not on each example in the training set. For each test point, we compare the confidence of the classifier versus the human correctness model and defer accordingly.

SelectivePrediction: we train the classifier using the cross entropy loss on all the data, for the rejector, we learn a single threshold on the validation set for the classifier confidence (probability of the predicted class) in order to maximize system accuracy.

\subsection{Training Details}

\begin{table}[H]
    \centering
        \caption{Training details for each dataset, we use the Adam optimizer \citep{kingma2014adam} and AdamW \citep{loshchilov2017decoupled} }
    \begin{tabular}{p{0.3 \textwidth}ccllp{0.4 \textwidth}}
    \toprule
         \textbf{Dataset} & Optimizer & Number of Epochs  & Learning Rate \\ 
         \toprule
         SyntheticData (ours) & Adam & 300 & 0.1\\
         CIFAR-K  &  Adam & 100 & 0.001 \\
         \midrule
         CIFAR-10H \citep{battleday2020capturing} & AdamW & 20 & 0.001 \\
         Imagenet-16H \citep{kerrigan2021combining} & Adam & 20 & 0.001 \\
         HateSpeech \citep{davidson2017automated} & Adam & 50 & 0.001 \\
         COMPASS \citep{dressel2018accuracy} & Adam &300 & 0.1\\
         NIH Chest X-ray  \citep{wang2017hospital,majkowska2020chest}
 & AdamW & 3 & 0.001 \\ \bottomrule
    \end{tabular}
    \label{tab:training_details}
\end{table}


\subsection{Synthetic Data} \label{apx:synth_data}

 We show in Figure \ref{fig:synthetic-res-uniform} the performance of the different methods with the same setup with the uniform data distribution.

\begin{figure}[H]
    \centering
          \includegraphics[scale = 0.7]{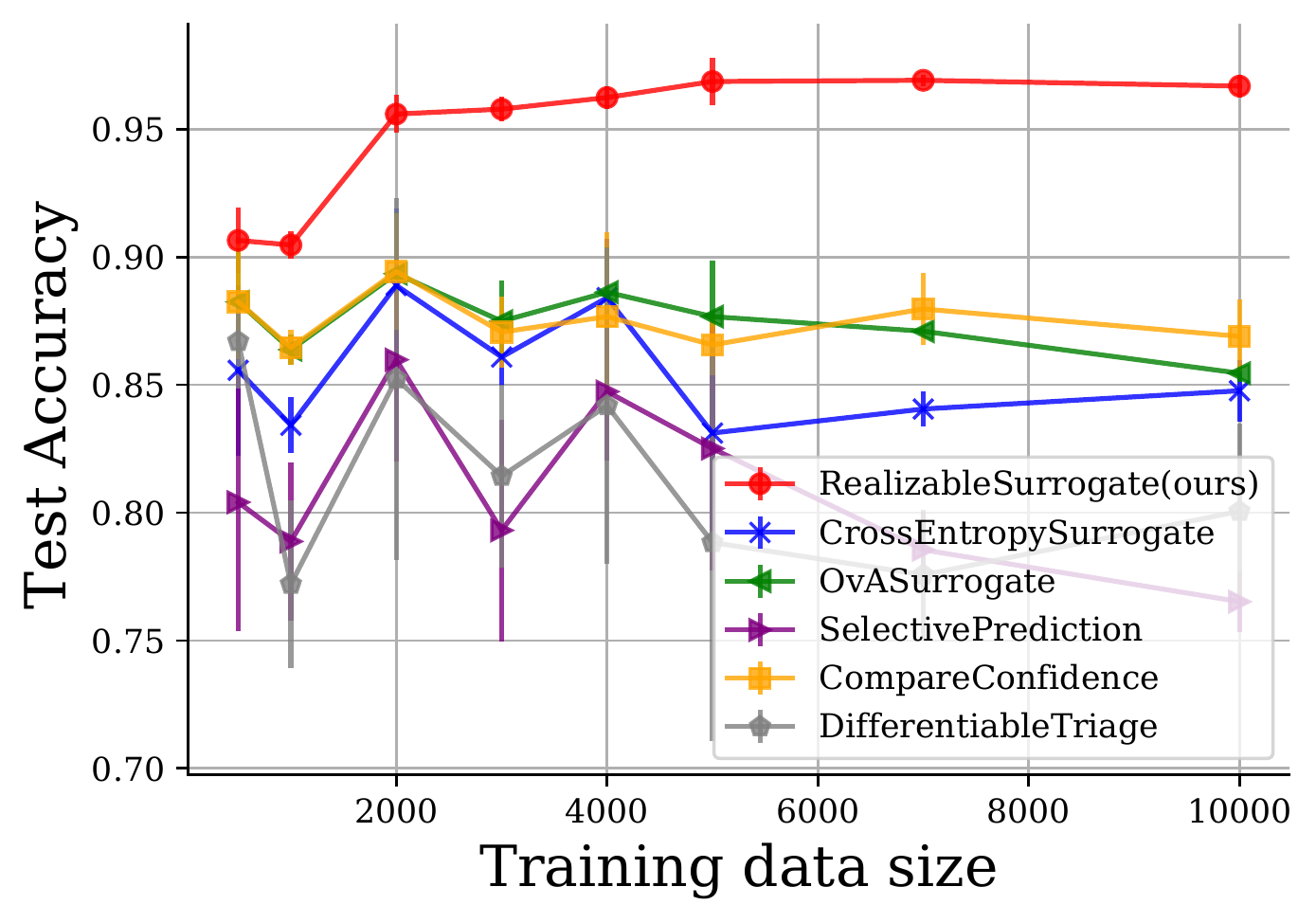}

    \caption{(Test performance of the different methods on realizable synthetic data as we increase the training data size and repeat the randomization over 10 trials to get standard errors on uniform data.}
        \label{fig:synthetic-res-uniform}

\end{figure}

We also experiment with making the data unrealizable by setting   ($d=10$, $p_m=0.1, p_{h0}=0.4, p_{h1}=0.1$, Gaussian distribution with 20 clusters) in Figure \ref{fig:synthetic-res-gaussain-unrez}. 

\begin{figure}[H]
    \centering
          \includegraphics[scale = 0.7]{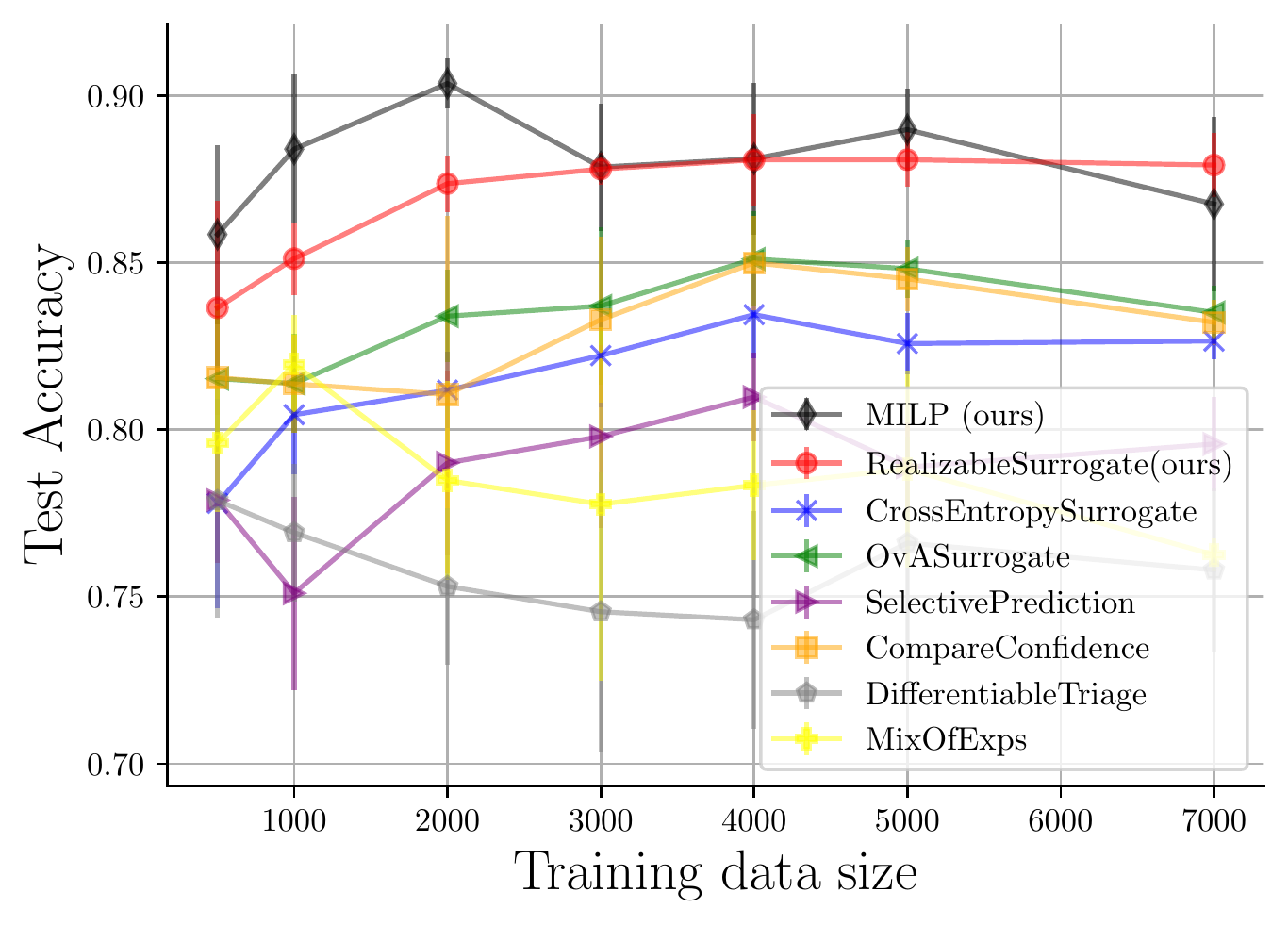}

    \caption{(Test performance of the different methods on unrealizable ($d=10$, $p_m=0.1, p_{h0}=0.4, p_{h1}=0.1$, Gaussian distribution with 20 clusters)  synthetic data as we increase the training data size.}
        \label{fig:synthetic-res-gaussain-unrez}

\end{figure}

We also show average run-times for the MILP on the synthetic data as we increase the dimension in Figure \ref{fig:synthetic-res-dimension} and as we increase the training data size in Figure \ref{fig:synthetic-res-size}. The distribution was uniform and realizable  with $p_m=0.0, p_{h0}=0.3, p_{h1}=0.0$. We observe that the run time increases with training set size which is the biggest bottleneck. The runtime also increases with dimension up until the dimension is of the same order as the number of training points, afterwards it is  faster for the MILP to find a 0 error solution.

\begin{figure}[t]
    \centering
    \begin{subfigure}{0.48\textwidth}
        \centering
          \includegraphics[width=\textwidth]{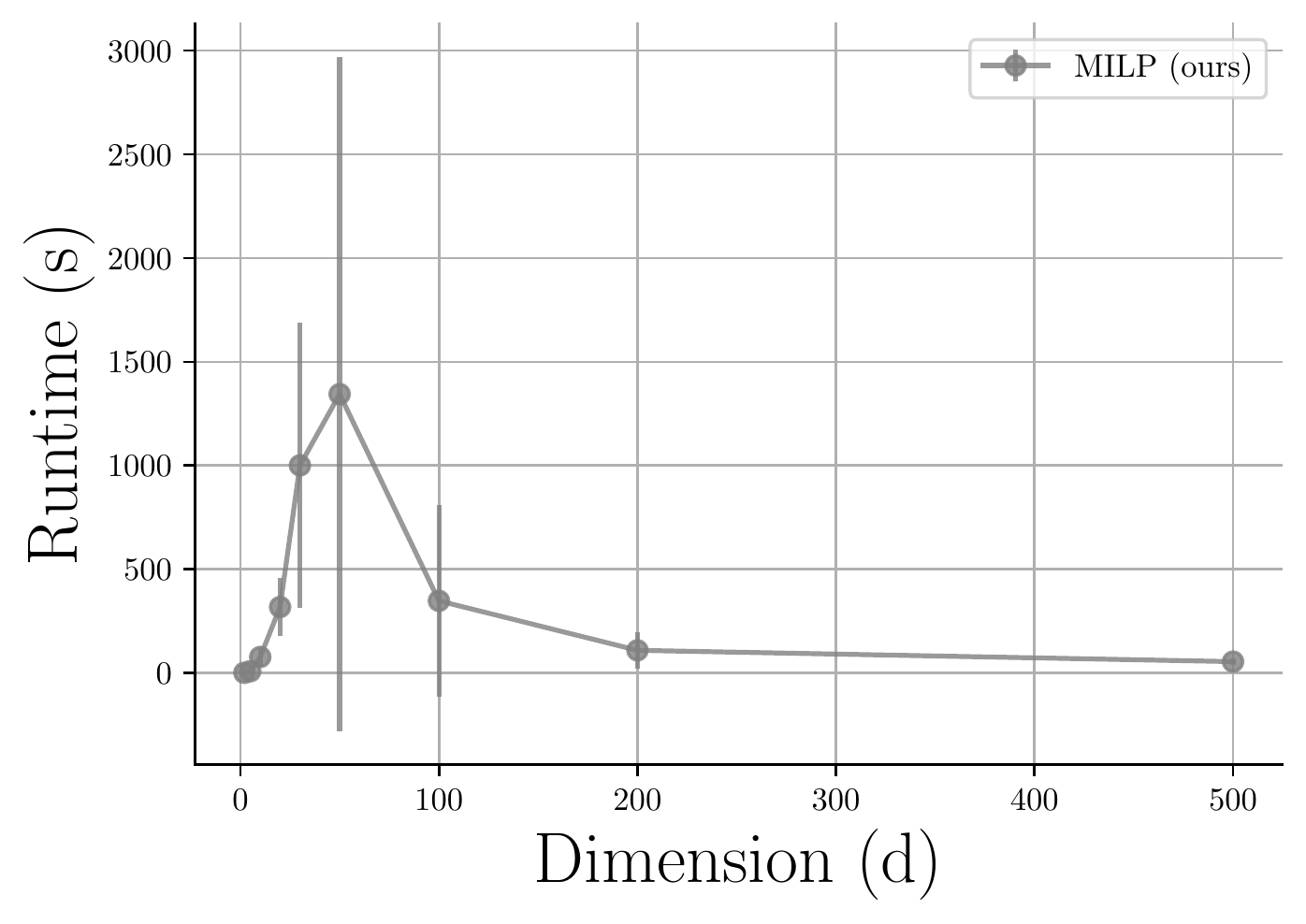}
        \subcaption{Runtime with increasing dimension, $n=1000$}
        \label{fig:synthetic-res-dimension}%
    \end{subfigure}\hfill%
    \begin{subfigure}{0.48\textwidth}
        \centering
          \includegraphics[width=\textwidth]{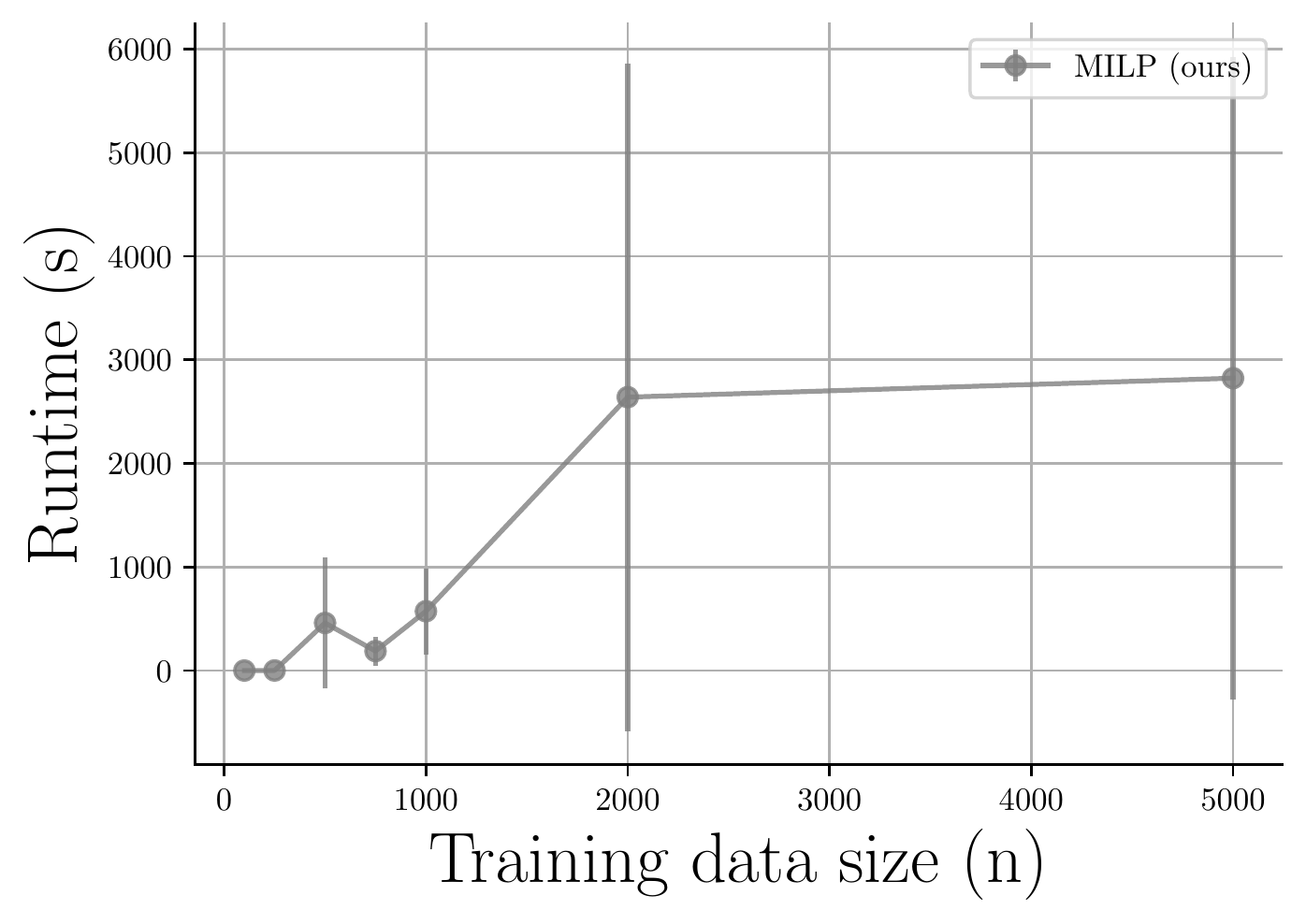}
        \subcaption{Runtime with increasing training data size, $d=30$}
        \label{fig:synthetic-res-size}%
    \end{subfigure}
    \caption{Runtime of the MILP on the realizable synthetic data with uniform data distribution. Note that the test accuracy of the MILP is demonstrated in Figure \ref{fig:synthetic-res} and the MILP always reaches 0 training error across the different data dimensions and training set sizes. }
            \vspace{-0em}
\end{figure}

\subsection{NIH Chest X-ray}

\begin{figure}[H]
    \centering
    \begin{subfigure}{0.48\textwidth}
        \centering
          \includegraphics[width=\textwidth]{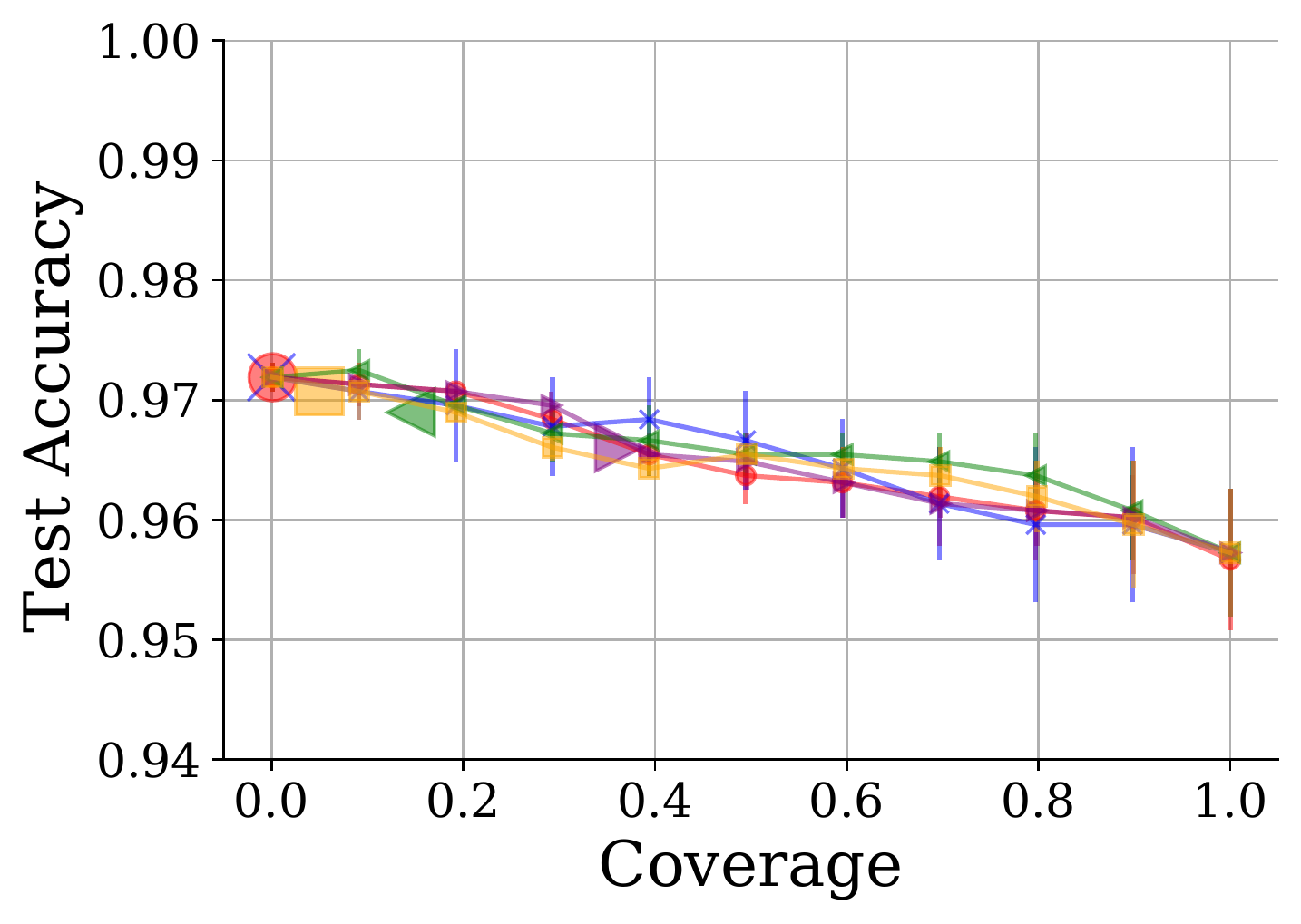}
        \subcaption{Fracture}
        \label{fig:chest0}%
    \end{subfigure}\hfill%
    \begin{subfigure}{0.48\textwidth}
        \centering
          \includegraphics[width=\textwidth]{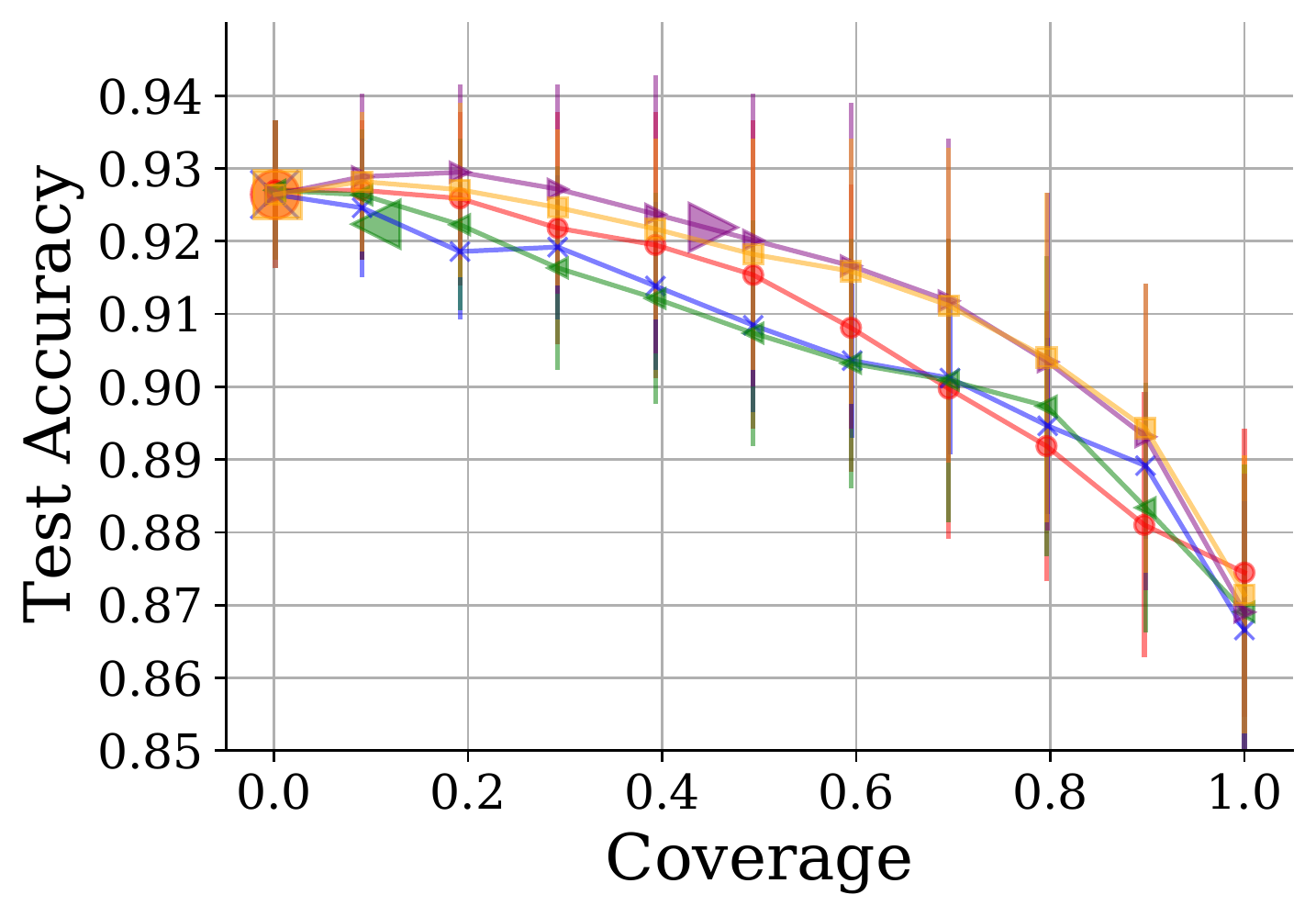}
        \subcaption{Nodule or Mass}
        \label{fig:chest3}%
    \end{subfigure}
    \caption{NIH Chest X-ray results on the two remaining tasks with the baselines and our method and red with circle markers. We see that all methods aren't able to obtain a performance of a human-AI team with better performance than the human, our method on both tasks defers to the human.  }
            \vspace{-0em}
\end{figure}

\subsection{CIFAR-10H}

\begin{figure}[H]
    \begin{subfigure}{0.5\textwidth}
        \centering
          \includegraphics[width=\textwidth]{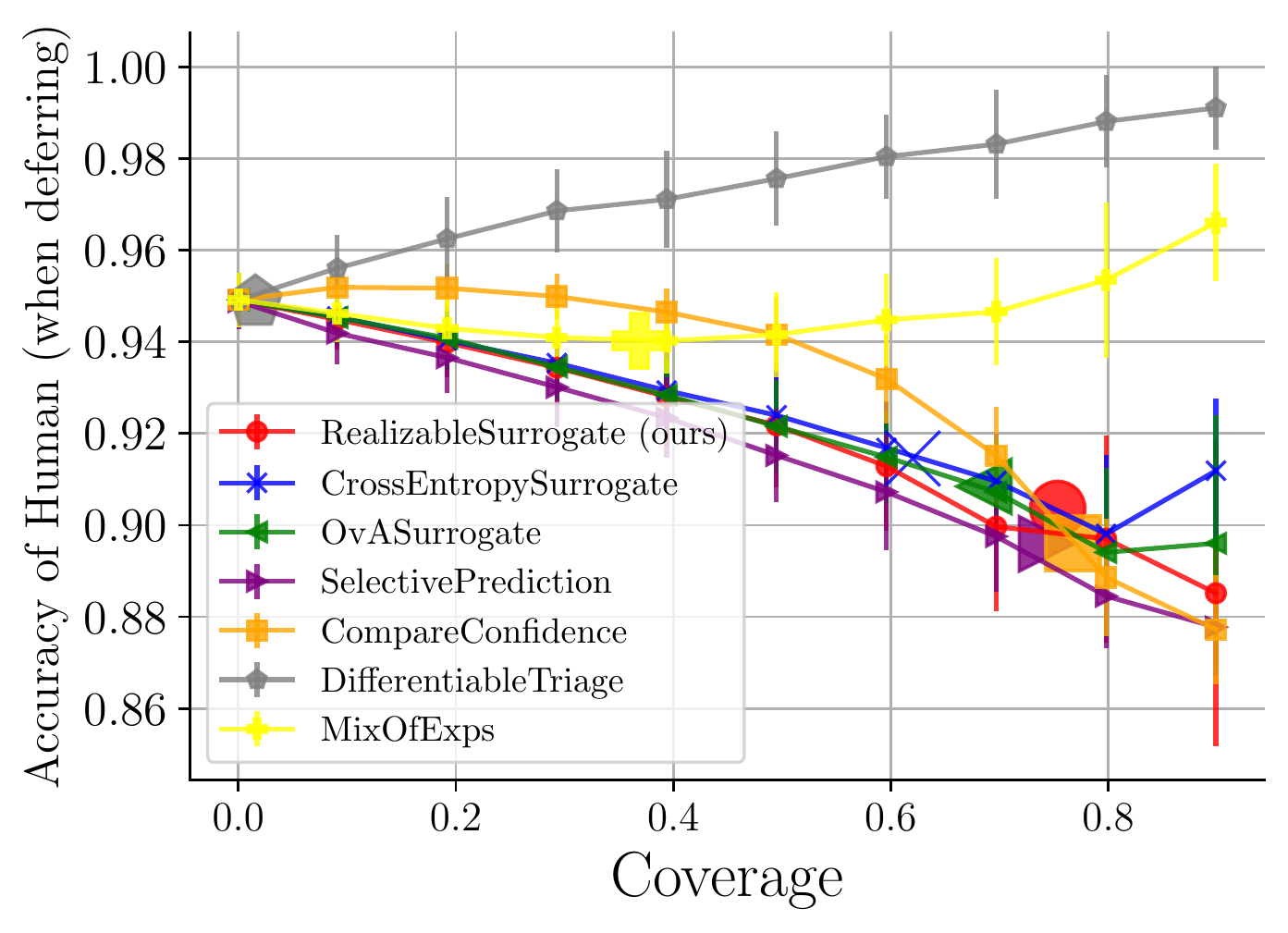}
        \subcaption{Accuracy on the examples deferred to human}
    \end{subfigure}%
    \begin{subfigure}{0.5\textwidth}
        \centering
          \includegraphics[width=\textwidth]{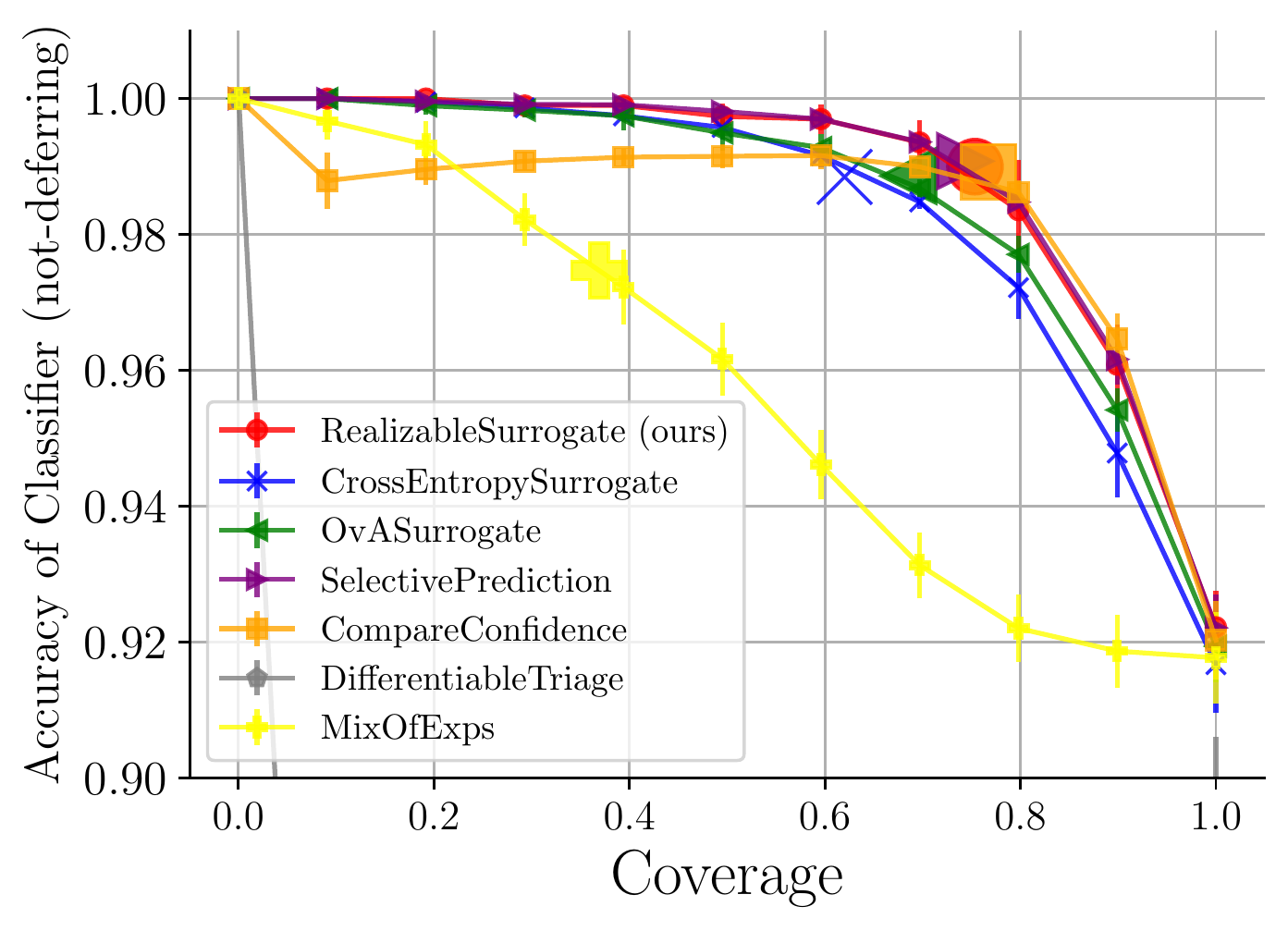}
        \subcaption{Accuracy on the examples not deferred, classifier predicts }
    \end{subfigure} 
 
    \caption{\looseness=-1 On CIFAR-10H, classifier accuracy on non-deferred set and human accuracy when deferred vs coverage (fraction of points where classifier predicts).  }
      \label{fig:cifar10h_more_results}
\end{figure}

\section{Deferred Proofs and Derivations}\label{apx:proofs}
\subsection{Related Work}\label{apx:ova_special_family}
We mentioned that the surrogate in \cite{verma2022calibrated} belongs to the family derived in \cite{charusaie2022sample}. 

This is established by setting $l_{\phi}(i,f(x))$ as follows \footnote{This was established by Yuzhou Cao.}:
\begin{equation}
    l_{\phi}(i,f(x)) = \begin{cases}
    \phi(g_y) + \sum_{y' \neq y} \phi (-g_{y'}), \quad & if \ y \in \mathcal{Y}  \\
    \phi(g_y) - \phi (-g_{y'}), \quad 
    \end{cases}
\end{equation}

\subsection{Section \ref{sec:linear_case} (Hardness)}
\subsubsection{Background and Definitions}
\paragraph{Realizable Intersection of Halfspaces.}
For our purposes, an instance $\cI$ of learning an intersection of halfspaces in the realizable setting is given by a finite dataset $\{(x_i,y_i)\}_{i=1}^n$, with $x_i \in \mathbb{R}^d$ and $y_i \in \{0,1\}$, such that there exist halfspaces $g_1^*: \mathbb{R}^d \to \{0,1\}$ and $g_2^*: \mathbb{R}^d \to \{0,1\}$
with zero error on the dataset:
\[
    err_{\cI}(g_1^*, g_2^*) := \frac{1}{n}\sum_{i}\mathbb{I}_{g_1^*(x_i) \wedge g_2^*(x_i) \ne y_i} = 0.
\]
We consider two related problems: finding halfspaces $(g_1,g_2)$ with \emph{exact} and \emph{weak} agreement.
\paragraph{Exact agreement.} Given an instance $\cI$ of realizable intersection of halfspaces, the exact agreement problem is to find a pair of halfspaces $(g_1,g_2)$ such that $g_1(x_i) \wedge g_2(x_i) = y_i$ for all $i \in \{1,\ldots, n\}$.

\paragraph{Weak agreement.}
Given an instance $\cI$ of realizable intersection of halfspaces, the weak agreement problem is to find a pair of halfspaces $(g_1,g_2)$ with error at most $1/2-\gamma$ for some $\gamma > 0$:
\[
err_{\cI}(g_1, g_2) := \frac{1}{n}\sum_{i}\mathbb{I}_{g_1(x_i) \wedge g_2(x_i) \ne y_i} \le \frac{1}{2}-\gamma.
\]
Note that there exists a pair $(g_1^*, g_2^*)$ with error 0 but the goal is just to obtain error $1/2-\gamma$.

Quite a bit is known about the hardness of the exact and weak agreement problems.
\begin{theorem*}[\citet{blum1988training} Theorem 1, rephrased]
The exact agreement problem is NP-hard.
\end{theorem*}
\begin{theorem*}[\citet{khot2011hardness} Theorem 2, rephrased]
There is no polynomial-time algorithm for the weak agreement problem unless $NP=RP$.
\end{theorem*}

We also consider finite-data versions of LWD-H:
\paragraph{Finite-data realizable LWD-H.}
An instance $\cJ$ of learning with deferral in the realizable setting is given by a finite dataset $\{(x_i,y_i,h_i)\}_{i=1}^n$, with $x_i \in \mathbb{R}^d$ and $y_i,h_i \in \{0,1\}$, such that there exist halfspaces $m^*: \mathbb{R}^d \to \{0,1\}$ and $r^*: \mathbb{R}^d \to \{0,1\}$
with zero error on the dataset:
\[
    err_{\cJ}(m^*, r^*) := \frac{1}{n}\sum_{i}\mathbb{I}_{r^*(x_i)=1}\mathbb{I}_{h_i\ne y_i} + \mathbb{I}_{r^*(x_i)=0}\mathbb{I}_{m^*(x_i) \ne y_i} = 0.
\]
As with intersection-of-halfspaces, we can consider finding halfspace classifier/rejector pairs $(m,r)$ with \emph{exact} and \emph{weak} agreement.
\paragraph{Exact agreement.} Given an instance $\cJ$ of realizable LWD-H, the exact agreement problem is to find a pair of halfspaces $(m,r)$ such that for all $i$, if $r(x_i) = 0$, $m(x_i) = y_i$, and if $r(x_i) = 1$, $h_i=y_i$.
That is, the error of the classifier/human system on the finite dataset is 0.
\paragraph{Weak agreement.} Given an instance $\cJ$ of realizable LWD-H, the weak agreement problem is to find a pair of halfspaces $(m,r)$ with error at most $1/2-\gamma$ for some $\gamma > 0$:
\[
err_{\cJ}(m, r) := \frac{1}{n}\sum_{i}\mathbb{I}_{r(x_i)=1}\mathbb{I}_{h_i\ne y_i} + \mathbb{I}_{r(x_i)=0}\mathbb{I}_{m(x_i)\ne y_i}\le \frac{1}{2}-\gamma.
\]

\subsubsection{Mapping between learning intersections and LWD-H}
We show how to turn an instance $\cI$ of realizable intersection of halfspaces into an instance of $\cJ$ of (finite-data) realizable LWD-H.
Given an arbitrary instance $\cI$ on dataset $\cD$, Lemma \ref{apdx:prop:error-relation} shows how to construct an instance $\cJ$ of LWD-H and a bijection $(g_1, g_2) \longleftrightarrow (m,r)$ such that for arbitrary halfspaces $(g_1,g_2)$, the error $err_{\cI}(g_1,g_2) = err_{\cJ}(m,r)$.
In particular, since we assumed $\cI$ is realizable and hence $\exists g_1^*, g_2^*$ with $err_{\cI}(g_1^*,g_2^*) = 0$, Lemma \ref{apdx:prop:error-relation} shows how to construct an instance $\cJ$ of LWD-H with $err_{\cJ}(m^*, r^*) = 0$.
This will allow us to reduce an arbitrary instance $\cI$ of realizable intersection of halfspaces to an instance $\cJ$ of realizable LWD-H.
Additionally, given an arbitrary classifier/rejector pair $(m,r)$ on this $\cJ$ with error $\epsilon$, Lemma \ref{apdx:prop:error-relation} shows how to map $(m,r)\to (g_1,g_2)$ with error $\epsilon$ on instance $\cI$.
\begin{lemma}
\label{apdx:prop:error-relation}
Consider an arbitrary instance $\cI$ of learning an intersection of halfspaces on a dataset $\cD = \{(x_i, y_i)\}_{i=1}^n$.
Define $\tilde \cD = \{(x_i,y_i,0)\}_{i=1}^n$.
This corresponds to an instance $\cJ$ of LWD-H where the ``human expert'' always outputs label 0.

Then:
\begin{enumerate}
    \item Consider two arbitrary halfspaces $g_1, g_2$ and set $m(x) = g_1(x)$, $r(x) = 1-g_2(x)$.
Note that $m$ and $r$ are also halfspaces. Then $err_{\cI}(g_1,g_2) = err_{\cJ}(m,r)$.
That is,
\[
\frac{1}{n}\sum_{(x_i,y_i) \in \cD} \mathbb{I}[g_1(x_i) \wedge g_2(x_i) \ne y_i] = \frac{1}{n}\sum_{(x_i, y_i, h_i) \in \tilde \cD} \left(\mathbb{I}_{r(x_i)=1}\mathbb{I}_{h_i\ne y_i} + \mathbb{I}_{r(x_i) = 0}\mathbb{I}_{m(x_i) \ne y_i}\right).
\]
\item Suppose $\cI$ is an instance of \emph{realizable} intersection of halfspaces. Then the instance $\cJ$ of LWD-H defined by the dataset $\tilde \cD$ is an instance of \emph{realizable} LWD-H. That is, there exists $(m^*,r^*)$ with $err_{\cJ}(m^*, r^*) = 0$.
\end{enumerate}
\end{lemma}
\begin{proof}
For part 1, recall that by definition:
\[
    \err_{\cJ}(m,r) = \frac{1}{n}\sum_{(x_i,h_i,y_i)\in \tilde \cD}\left(\mathbb{I}_{r(x_i) = 1}\mathbb{I}_{h_i \ne y_i} + \mathbb{I}_{r(x_i)=0}\mathbb{I}_{m(x_i)\ne y_i}\right).
\]
Since $h_i = 0$ for all $i$, this is equal to
\[
    \frac{1}{n}\sum_{i}\left(\mathbb{I}_{r(x_i) = 1}\mathbb{I}_{y_i = 1} + \mathbb{I}_{r(x_i)=0}\mathbb{I}_{m(x_i)\ne y_i}\right).
\]
Using $r(x) = 1-g_2(x)$ and $m(x) = g_1(x)$, this simplifies further to:
\begin{align}\label{apdx:eq:reduction-error}
    \frac{1}{n}\sum_{i}\left(\mathbb{I}_{g_2(x_i) = 0}\mathbb{I}_{y_i = 1} + \mathbb{I}_{g_2(x_i) = 1}\mathbb{I}_{g_1(x_i)\ne y_i}\right).
\end{align}
Consider the error of $err_{\cI}(g_1,g_2)$.
The model makes a mistake if $g_2(x)=0$ and $y(x)=1$, $g_2(x) = g_1(x) = 1$ and $y=0$, or $g_2(x) = 1, g_1(x) = 0$, and $y=1$.
The first case is $\mathbb{I}_{g_2(x) = 0}\mathbb{I}_{y = 1}$ and the latter two cases can be expressed as $\mathbb{I}_{g_2(x)=1}\mathbb{I}_{g_1(x) \ne y}$.
Hence 
\[
    err_{\cI}(g_1,g_2) = \frac{1}{n}\sum_{(x_i,y_i) \in \tilde \cD} \mathbb{I}[g_1(x_i) \wedge g_2(x_i) \ne y_i] = \frac{1}{n}\sum_{i}\left(\mathbb{I}_{g_2(x_i) = 0}\mathbb{I}_{y_i = 1} + \mathbb{I}_{g_2(x_i) = 1}\mathbb{I}_{g_1(x_i)\ne y_i}\right),
\]
which is equal to \eqref{apdx:eq:reduction-error}, so $err_{\cI}(g_1,g_2) = \err_{\cJ}(m,r)$.

For part 2, we assumed that $\cI$ was realizable, so there exists $g_1^*$, $g_2^*$ with $err_{\cI}(g_1^*, g_2^*) = 0$. Applying part 1 yields $m^*, r^*$ such that $err_{\cJ}(m^*, r^*) = 0$. Hence $\cJ$ is an instance of realizable LWD-H.
\end{proof}

Lemma \ref{apdx:prop:error-relation} takes an instance $\cI$ of learning an intersection of halfspaces and constructs an instance $\cJ$ of LWD-H such that there is an error-preserving bijection between solutions of $\cI$ and solutions of $\cJ$.
This allows us to easily apply the existing hardness results for learning a realizable intersection of halfspaces, since if $\cI$ is realizable then so is $\cJ$.

\subsubsection{Hardness results for LWD-H}
\begin{theorem}
There is no polynomial-time algorithm for solving the exact agreement problem for LWD-H unless P=NP.
\end{theorem}
\begin{proof}
Suppose there exists a polytime algorithm $\cA$ for solving exact agreement on realizable LWD-H.
Consider an arbitrary instance $\cI$ of learning a realizable intersection of halfspaces.
Lemma \ref{apdx:prop:error-relation} shows how to construct an instance $\cJ$ of realizable LWD-H. Run Algorithm $\cA$ on $\cJ$ to obtain halfspaces $(m,r)$ with $err_{\cJ}(m,r) = 0$. Set $g_1 = m$, $g_2 = 1-r$. Lemma \ref{apdx:prop:error-relation} guarantees that $err_{\cI}(g_1,g_2) = 0$. Hence, $\cA$ is a polynomial-time algorithm for exact agreement for realizable intersection of halfspaces.
\citet{blum1988training} shows that there is no polynomial-time algorithm for exact agreement for realizable intersection of halfspaces unless $P=NP$.
\end{proof}
\begin{corollary}
There is no efficient, proper PAC learner for realizable LWD-H unless $NP=RP$.
\end{corollary}
\begin{proof}[Proof sketch]
Suppose $\cA$ is an efficient proper PAC learner for realizable LWD-H, so for any distribution $\cD$, any $\epsilon > 0$, $\delta > 0$, given $poly(1/\delta, 1/\epsilon)$ samples from $\cD$, $\cA$ outputs a pair of halfspaces $(m,r)$ with (population) system error at most $\epsilon$ in time $poly(1/\epsilon, 1/\delta)$.

Now let $\cD$ be the uniform distribution over a dataset of $n$ points $\{(x_i,y_i,h_i)\}_{i=1}^n$. Set $\epsilon = 1/(2n)$ and $\delta=1/100$ and run $\cA$.
With probability at least $1-\delta$ $\cA$ outputs $(m,r)$ with error at most $1/(2n)$. Of course, if $(m,r)$ has error at most $1/(2n)$ it must have error 0.
This gives a randomized algorithm for solving the exact agreement problem for realizable finite-data LWD-H.
\end{proof}

These results show that exact agreement, and thus exact proper PAC learning, are hard. Next we consider the hardness of weak agreement.

\textbf{Theorem \ref{prop:hardness}} \textit{
Let $\epsilon > 0$ be an arbitrarily small constant and suppose we have an instance $\cJ$ of realizable LWD-H. So we have data $\cD = \{(x_i,y_i,h_i)\}_{i=1}^n$, where $x_i\in \mathbb{R}^d, y_i, h_i \in \{0,1\}$, and there exist halfspaces $m^*, r^*$ with zero loss on $\cD$:
\[
    err_{\cJ}(m^*,r^*) := \frac{1}{n}\sum_{i}\left(\mathbb{I}_{r^*(x_i) = 1}\mathbb{I}_{h_i \ne y_i} + \mathbb{I}_{r^*(x_i)=0}\mathbb{I}_{m^*(x_i)\ne y_i}\right) = 0
\]
Then there is no polynomial-time algorithm to find a classifier-rejector pair $(\hat{m}, \hat{r})$ with error $1/2 -\epsilon$, i.e.:
\[
    \frac{1}{n}\sum_{i}\left(\mathbb{I}_{\hat r(x_i)=1}\mathbb{I}_{h_i \ne y_i} + \mathbb{I}_{\hat r(x_i)=0}\mathbb{I}_{\hat m(x_i)\ne y_i}\right)\le \frac{1}{2} - \epsilon
\]
 unless $NP=RP$.
}
\begin{proof}
Suppose there exists a polynomial-time algorithm $\cA$ and a $\gamma > 0$ such that given an instance $\cJ$ of realizable LWD-H, $\cA$ returns a pair $(\hat m, \hat r)$ with error $err_{\cJ}(\hat m, \hat r)$ at most $1/2-\gamma$.
Consider an arbitrary instance $\cI$ of realizable intersection of halfspaces.
Lemma \ref{apdx:prop:error-relation} shows how to reduce $\cI$ to an instance $\cJ$ of realizable LWD-H.
Run Algorithm $\cA$ on $\cJ$ to obtain a pair of halfspace $(\hat m, \hat r)$ with error at most $err_{\cJ}(\hat m, \hat r) \le 1/2-\gamma$. 
Lemma \ref{apdx:prop:error-relation} guarantees that $g_1 = \hat m, g_2 = 1-\hat r$ satisfy $err_{\cI}(g_1,g_2) \le 1/2-\gamma$.
Hence $\cA$ gives a deterministic algorithm for solving the weak agreement problem for realizable intersection of halfspaces.
\citet[Theorem 4]{khot2011hardness} construct an algorithm/reduction showing that if we can efficiently solve weak agreement for realizable intersection of halfspaces, then Smooth Label Cover is in $RP$, but Smooth Label Cover is an NP-hard problem \citep[][Theorem 3]{khot2011hardness}.
Hence there is no polynomial-time algorithm to find a classifier-rejector pair $(\hat{m}, \hat{r})$ with error $1/2 -\epsilon$ unless $NP=RP$.
\end{proof}
\begin{corollary}
There is no efficient, proper, \emph{weak} PAC-learner for realizable LWD-H unless $NP=BPP$.
\end{corollary}
\begin{proof}
Given a distribution $\cD$ over points $(x,y,h)$, $x\in\mathbb{R}^d$, $y,h\in \{0,1\}$ and halfspaces $(m,r)$, let
\[
err_{\cD}(m,r) := \mathbb{P}_{(x,y,h)\sim \cD}[r(x)=1 \wedge h\ne y \vee r(x) = 0 \wedge m(x) \ne y].
\]
This is identical to the \emph{system loss} \eqref{eq:01_reject_loss} on distribution $\cD$.
Suppose there exists an efficient, proper, weak PAC-learner for realizable LWD-H.
I.e., there exists some $\gamma$ such that for any distribution $\cD$, under the guarantee that $\exists (m^*,r^*)$ with $err_{\cD}(m^*, r^*) = 0$, given access to $poly(1/\delta)$ samples from $\cD$, with probability at least $1-\delta$, $\cA$ returns a pair $(m,r)$ with $err_{\cD}(m,r) \le \frac{1}{2}-\gamma$ in $poly(1/\delta)$ time.

By combining Lemma \ref{apdx:prop:error-relation} with the randomized reduction of \citet{khot2011hardness}, we can use $\cA$ to construct an algorithm that implies Smooth Label Cover is in $BPP$.
The definition of Smooth Label Cover is not important for our purposes beyond the following two results:
\begin{theorem}{\citep[Theorem 3]{khot2011hardness}}
\label{thm:smooth-label-cover-hardness}
For any constant $t$ and arbitrarily small constants $\mu, \nu, \eta > 0$, there exist constants $k$ and $m$ such that given an instance
$\cL$ of $\textrm{Smooth-Label-Cover}(t, \mu , \nu , k, m)$ it is NP-hard to distinguish between the following two cases:
\begin{itemize}
    \item YES Case/Completeness: There is a labeling to the vertices of $\cL$ which satisfies all the edges.
    \item NO Case/Soundness: No labeling to the vertices of $\cL$ satisfies more than $\eta$ fraction of the edges.
\end{itemize}
\end{theorem}

\begin{theorem}{\citep[Theorem 4]{khot2011hardness}}
\label{thm:smooth-label-cover-reduction}
For any constant $\gamma > 0$ and integer $l > 0$, there is a randomized polynomial time reduction from an instance $\cL$ of $\textrm{Smooth-Label-Cover}(t, \mu , \nu , k, m)$ to an instance $\cI$ of Realizble Intersection of Halfspaces for appropriately chosen parameters $(t, \mu , \nu)$ and soundness $\eta$,
such that
\begin{itemize}
    \item YES Case/Completeness: If $\cL$ is a YES instance, then there is an intersection of two halfspaces which correctly classifies all the points in instance $\cI$.
    \item NO Case/Soundness: If $\cL$ is a NO instance, then with probability at least 9/10, there is no function of up to $l$ halfspaces that correctly classifies more than $1/2 + \gamma$ fraction of points in instance $\cI$.
\end{itemize}
\end{theorem}

For our case, we can use Lemma \ref{apdx:prop:error-relation} to further reduce the instance $\cI$ constructed by Theorem \ref{thm:smooth-label-cover-reduction} to an instance $\cJ$ of LWD-H, then run the weak PAC-learner $\cA$ on $\cJ$.
If $\cA$ outputs a pair of halfspaces $(m,r)$ with error at most $1/2-\gamma$, we output YES.
Otherwise we output NO.

If $\cI$ is a realizable instance, $\cA$ returns a pair of halfspaces with error at most $1/2-\gamma$ with probability at least $1-\delta$.
On the other hand, if $\cI$ is not weakly realizable (w.r.) (i.e., there is no function of up to $l$ halfspaces that correctly classifies more than a $1/2 + \gamma$ fraction of points in $\cI$), then clearly $\cA$ never returns a good pair of halfspaces, since no such pair exists.
Therefore:
\begin{align*}
\bP(\text{YES}|\cL \text{ YES}) &= \bP(\text{YES} | \cI \text{ realizable})\bP(\cI \text{ realizable} | \cL \text{ YES})\\
&= (1-\delta)\cdot 1
\end{align*}
\begin{align*}
\bP(\text{NO}|\cL \text{ NO}) &= \bP(\text{NO} | \cI \text{ w.r.})\bP(\cI \text{ w.r.} | \cL \text{ NO}) + \bP(\text{NO} | \cI \text{ not w.r.})\bP(\cI \text{ not w.r.} | \cL \text{ NO})\\
&\ge \bP(\text{NO} | \cI \text{ not w.r.})\bP(\cI \text{ not w.r.} | \cL \text{ NO})\\
&\ge \bP(\text{NO} | \cI \text{ not w.r.})\frac{9}{10}\\
&= \frac{9}{10}.
\end{align*}

Hence we can use $\cA$ to construct an algorithm for a Smooth-Label-Cover instance $\cL$ that outputs YES when $\cL$ is a YES with probability at least $(1-\delta)$, and outputs NO when $\cL$ is a NO with probability at least 9/10.
Since we assumed $\cA$ runs in $poly(1/\delta)$, this implies Smooth Label Cover is in $BPP$.
Together with Theorem \ref{thm:smooth-label-cover-hardness}, this shows that there is no efficient, proper, weak PAC learner for realizable LWD-H unless $NP=BPP$.
\end{proof}

Finally, we show that when realizability is violated, there is no efficient algorithm for weak agreement.

\textbf{Corollary \ref{th:hardness_extension}} (formal). \textit{
Let $\delta, \epsilon > 0$ be arbitrarily small constants. Then, given a set of points $\{(x_i,y_i,h_i)\}$ with $x_i \in \mathbb{R}^d$, $y_i, h_i \in \{0,1\}$ with a guarantee that there is a classifier/rejector pair $(m^*,r^*)$ that classifies a $1 - \delta$ fraction of points correctly, there is no polynomial time algorithm to find a classifier-rejector pair that classifies $\frac{1}{2} + \epsilon$
fraction of points correctly unless P = NP.
}
\begin{proof}
This is a simple reduction from learning a single halfspace in the presence of noise, which is hard by the following result:
\begin{theorem*}{(\citet{guruswami2009hardness}, see also \citet[Theorem 1]{khot2011hardness})}
Let $\delta, \epsilon > 0$ be arbitrarily small constants. Then, given a set of labeled points $\{(x_i,y_i)\}$ in $\mathbb{R}^d$ with a guarantee that there is a
halfspace that classifies $1 -\delta$ fraction of points correctly, there is no polynomial time algorithm to find a halfspace that classifies $1/2 + \epsilon$ fraction of points correctly, unless P = NP.
\end{theorem*}
Suppose we have an algorithm $\cA$ for solving LWD-H in the presence of noise.
In particular, there exists some $\epsilon > 0, \delta > 0$ such that under the guarantee that there exists an $(m^*,r^*)$ pair with error at most $\delta$, $\cA$ returns an $(m,r)$ pair with error at most $\frac{1}{2}-\epsilon$.

Consider an instance $\cI$ of learning a single halfspace in the presence of noise defined 
by a dataset $\cD = \{(x_i,y_i)\}_{i=1}^n$, such that there exists a halfspace $c$ with error at most $\delta$ on $\cD$.
From $\cD$, construct the dataset $\tilde \cD = \{(x_i, y_i, 1-y_i)\}_{i=1}^n$. 
This is an instance $\cJ$ of LWD-H where the ``human expert'' is always wrong.
Note that $(c,0)$ is a classifier/rejector pair with error at most $\delta$ on $\tilde \cD$, so $\cJ$ is an instance of LWD-H with noise level $\delta$.
Run algorithm $\cA$ on $\cJ$ with parameter $\epsilon$ to obtain an $(m,r)$ pair with $err_{\cJ}(m,r) = 1/2 - \epsilon$.
Then:
\begin{align*}
1/2 - \epsilon \ge err_{\cJ}(m,r) &= \frac{1}{n}\sum_{i}\mathbb{I}_{r(x_i)=1}\mathbb{I}_{h_i\ne y_i} + \mathbb{I}_{r(x_i)=0}\mathbb{I}_{m(x_i)\ne y_i}\\
&= \frac{1}{n}\left(\sum_{i: r(x_i)=1}\mathbb{I}_{h_i\ne y_i} + \sum_{i: r(x_i)=0}\mathbb{I}_{m(x_i)\ne y_i}\right)\\
&\ge \frac{1}{n}\left(\sum_{i: r(x_i)=1}\mathbb{I}_{m(x_i)\ne y_i} + \sum_{i: r(x_i)=0}\mathbb{I}_{m(x_i)\ne y_i}\right)\\
&= \frac{1}{n}\sum_{i}\mathbb{I}_{m(x_i)\ne y_i}\\
&= err_{\cI}(m),
\end{align*}
where the inequality is because we constructed $\tilde \cD$ such that $\mathbb{I}_{h_i \ne y_i} = 1$ for all $i$.
Therefore, there exists a $\delta$ and $\epsilon$ for which, given a dataset and the guarantee that there exists a halfspace with error at most $\delta$, we can output a halfspace with error at most $1/2 - \epsilon$.
Combining this with the Theorem above shows that if $\cA$ runs in polynomial time, $P=NP$.
\end{proof}

\subsection{Section \ref{sec:milp} (MILP)}

\textbf{Proposition 1.}\textit{ \noindent For any expert $H$ and data distribution $\mathbf{P}$ over $\mathcal{X} \times \mathcal{Y}$ that satisfies Assumption  \ref{ass: margin}, let $0<\delta<\frac{1}{2}$, then  with probability at least $1-\delta$, the following holds for the empirical minimizers $(\hat{m}^*,\hat{r}^*)$ obtained by the MILP:
\begin{align*} 
     \sysL(\hat{m}^*,\hat{r}^*) &\leq    \sysLhat(\hat{m}^*,\hat{r}^*)  \frac{(K_m + K_r) d \sqrt{2 \log d} + 10 \sqrt{\log(2/\delta)}}{\sqrt{n \bP(H(Z) \neq Y)}}    
\end{align*}
}
\begin{proof}
We first start by recalling Theorem 2 in \cite{mozannar2020consistent}:

\begin{align}
    \sysL(\hat{m}^*,\hat{r}^*) &\leq  \sysLhat(\hat{m}^*,\hat{r}^*) + \mathfrak{R}_n(\mathcal{M}) +  \mathfrak{R}_{n}(\mathcal{R})  + \mathfrak{R}_{n \bP(H(Z) \neq Y)/2}(\mathcal{R})  \nonumber \\
    & + 2\sqrt{\frac{\log{(\frac{2}{\delta})}}{2n}} +\frac{\bP(H(Z)\neq Y)}{2}  \exp\left(- \frac{n \bP(H(Z) \neq Y)}{8} \right)  \label{eq:gen_bound_mozannar2020}
\end{align}

Note that here we avoid going through the optimal solution and just relate distribution performance to the training performance.

In the bound \eqref{eq:gen_bound_mozannar2020}, $ \mathfrak{R}_n(\mathcal{M})$ and $ \mathfrak{R}_n(\mathcal{R})$ denote the Rademacher complexity of a halfspace in $d$ dimensions where the infinity norm of each element in the halfspace is constrained by $K_m$ and $K_r$ respectively. Let us now compute this Rademacher complexity, inspired by \cite{rademacherlinear}:

\begin{align*}
 \mathfrak{R}_n(\mathcal{M}) &= \frac{1}{n} \E \left[ \sup_{M: ||M||_\infty \leq K_m} \sum_{i=1}^n \epsilon_i M^\top x_i \right]\\
 &\leq  \frac{1}{n} \E \left[ \sup_{M: ||M||_1 \leq d K_m} M^\top\sum_{i=1}^n  \epsilon_i x_i \right] \quad \textrm{(since  $||M||_1 \leq d ||M||_\infty$)}\\  
 &= \frac{d K_m}{n} \E \left[   \sum_{i=1}^n  ||\epsilon_i x_i ||_{\infty} \right]  \\
  &= \frac{d K_m}{n} \E \left[   \sup_j \sum_{i=1}^n  \epsilon_i [x_i]_j \right]  \\
  & \leq \frac{d K_m \sqrt{2 \log d}}{n}    \sup_j  \sqrt{\sum_{i=1}^n  [x_i]^2_j} \quad \textrm{(Massart's finite lemma on $x_{ij}$)}   \\
  & \leq \frac{d K_m \sqrt{2 \log d}}{\sqrt{n}} \quad \textrm{(assume $||x_i||_1 \leq 1$ for all $i$ )}
\end{align*}

Let us use the Rademacher complexity calculation in the bound to get:
\begin{align*}
   \sysL(\hat{m}^*,\hat{r}^*) &\leq  \sysLhat(\hat{m}^*,\hat{r}^*) +  \frac{d K_m \sqrt{2 \log d}}{\sqrt{n}} +   \frac{d K_r \sqrt{2 \log d}}{\sqrt{n}} +  \frac{d K_m \sqrt{2 \log d}}{\sqrt{n \bP(H(Z) \neq Y)}} \nonumber \\
    & + 2\sqrt{\frac{\log{(\frac{2}{\delta})}}{2n}} +\frac{\bP(H(Z)\neq Y)}{2}  \exp\left(- \frac{n \bP(H(Z) \neq Y)}{8} \right)  
\end{align*}
note that $\frac{\bP(H(Z)\neq Y)}{2}  \exp\left(- \frac{n \bP(H(Z) \neq Y)}{8} \right)$ is a term that does not depend on the optimization and shrinks much faster than $\frac{8}{\sqrt{n \bP(H(Z) \neq Y)}}$, so that we can summarize things as:

\begin{align}
     \sysL(\hat{m}^*,\hat{r}^*) &\leq \sysLhat(\hat{m}^*,\hat{r}^*) +  \frac{(K_m + K_r) d \sqrt{2 \log d} + 10 \sqrt{\log(2/\delta)}}{\sqrt{n \bP(H(Z) \neq Y)}}    
\end{align}

\end{proof}

\subsection{Section \ref{sec:new_method} (\realizablesurrogate)} \label{apx:proffs_sec_realiz}

\textbf{Theorem 2.}\textit{ \noindent The \emph{\realizablesurrogatenosp} $L_{RS}$ is a  realizable $(\mathcal{M},\mathcal{R})$-consistent  surrogate for $L_{\mathrm{def}}^{0{-}1}$ for model classes closed under scaling, and satisfies $L_{\mathrm{def}}^{0{-}1}(m,r) \le L_{RS}(m,r)$ for all $(m,r)$.
}
\begin{proof}
Let us recall the \realizablesurrogate loss pointwise:
\begin{equation}
 \label{eq:proposed_RS_loss}
 L_{RS}(\mathbf{g},x,y,h) = -2 \log\left(\frac{\exp(g_{y}(x)) + \bI_{h = y} \exp(g_{\bot}(x)) }{\sum_{y' \in \mathcal{Y} \cup  \bot}\exp(g_{y'}(x))} \right) 
\end{equation}
where $\mathbf{g}=\{g_{i}\}_{ i \in  \mathcal{Y} \cup \bot} $.
Recall that the classifier and rejector are defined as: $m(x) = \arg \max_{y \in \mathcal{Y}}g_y(x)$ and  $r(x)= \bI_{\max_{y \in \mathcal{Y}}g_y(x) \leq g_\bot(x) }$.

We first prove that for every point, the \realizablesurrogate loss upper bounds the system 0-1 error: $L_{\mathrm{def}}^{0{-}1}(m,r,x,y,h) \le  L_{RS}(\mathbf{g},x,y,h)$:

\begin{enumerate}
    \item \textbf{Case 1:} consider $r(x)=0$ (classifier predicts):
    \begin{enumerate}
        \item \textbf{Case 1a:} if the classifier is incorrect, $\bI_{m(x) \neq y}=1$:
        \begin{enumerate}
            \item \textbf{Case 1ai:} If the human is incorrect, $\bI_{h = y}=0$:
            
            then the loss is:$
 -2 \log\left(\frac{\exp(g_{y}(x))  }{\sum_{y' \in \mathcal{Y} \cup  \bot}\exp(g_{y'}(x))} \right) $, we know since the classifier is incorrect, then it must be that $\frac{\exp(g_{y}(x))  }{\sum_{y' \in \mathcal{Y} \cup  \bot}\exp(g_{y'}(x))} \leq 0.5$ (since $g_{y}$ is not the max), thus the loss  is greater than $2$ (log is base 2), and the 0-1 loss is 1 in this case.
 \item \textbf{Case 1aii}: if the human is correct then $\bI_{h=y}=1$:
 
  then the loss is:$
 -2 \log\left(\frac{\exp(g_{y}(x)) + \exp(g_{\bot}(x))  }{\sum_{y' \in \mathcal{Y} \cup  \bot}\exp(g_{y'}(x))} \right) $, we know since the classifier is incorrect, then it must be that $\frac{\exp(g_{y}(x))  }{\sum_{y' \in \mathcal{Y} \cup  \bot}\exp(g_{y'}(x))} + \frac{\exp(g_{\bot}(x))  }{\sum_{y' \in \mathcal{Y} \cup  \bot}\exp(g_{y'}(x))}  < 2/3$ since $g_{y}$ is not the max neither is $g_{\bot}$, otherwise if the sum of these two fractions is greater than 2/3, then $\max_i \frac{\exp(g_{i}(x))  }{\sum_{y' \in \mathcal{Y} \cup  \bot}\exp(g_{y'}(x))} < 1/3 $ then the maximum must be one of $y$ or $\bot$ which is a contradiction. Finally, the loss is greater then $-2\log(2/3) = 1.17$ which is greater than $1$.
 
        \end{enumerate}
        \item \textbf{Case 1b:} if the classifier is correct $\bI_{m(x)=y}=1$, then the 0-1 error is 0, since the \realizablesurrogate loss is $\geq 0$ then it is an upper bound.
    \end{enumerate}
    \item \textbf{Case 2:} consider $r(x)=1$ (human predicts):
    \begin{enumerate}
        \item \textbf{Case 2a:} if the human is correct then $\bI_{h=y}=1$:
        
        then the 0-1 error is 0, since the \realizablesurrogate loss is $\geq 0$ then it is an upper bound.
        
        \item if the human is incorrect then $\bI_{h=y}=0$:
        
   the loss is      $
 -2 \log\left(\frac{\exp(g_{y}(x))  }{\sum_{y' \in \mathcal{Y} \cup  \bot}\exp(g_{y'}(x))} \right) $, we know since we defer, then it must be that $\frac{\exp(g_{y}(x))  }{\sum_{y' \in \mathcal{Y} \cup  \bot}\exp(g_{y'}(x))} \leq 0.5$ (since $g_{y}$ is not the max), thus the loss  is greater than $2$ (log is base 2), and the 0-1 loss is 1 in this case. 
    \end{enumerate}
\end{enumerate}
this concludes the proof of the upper bound.

We now prove that $L_{RS}$ is a realizable-consistent loss function. 

Consider a data distribution and a human under which there exists  $m^*,r^* \in \mathcal{M} \times \mathcal{R}$ that have zero error $L_{\mathrm{def}}^{0{-}1}(m^*,r^*)=0$. Associated with $m^*,r^*$, is a set of functions $\mathbf{g}^* \in \mathcal{G}$ that give rise to $m^*,r^*$. Let $\hat{\mathbf{g}}$ be the minimizer of the surrogate loss $L_{RS}$ and the associated classifier and rejector be $\hat{m},\hat{r}$. 

We now upper bound the 0-1 loss of the pair $\hat{m},\hat{r}$. Let $u \in \mathbb{R}$ be any real number:

\begin{align}
\nonumber
&L_{\mathrm{def}}^{0{-}1}(\hat{m},\hat{r}) \\&\leq \nonumber  L_{RS}(\hat{m},\hat{r}) \quad \textrm{ (loss is upper bound)}\\ \nonumber
&\leq L_{RS}(um^*,ur^*)  \quad \textrm{(since $\hat{m}, \hat{r}$ is optimal for $L_{RS}$ and $\cM \times \cR$ is closed under scaling)}\\ 
&= \bE[ L_{RS}(um^*,ur^*,x,y,h)|r^*=1] \bP(r^*=1) +   \bE[ L_{RS}(um^*,ur^*,x,y,h)|r^*=0] \bP(r^*=0) \label{eq:decompisition}
\end{align} 
Let us investigate the two terms in equation \eqref{eq:decompisition}.

The first term is when $r^*=1$, then we must have $g_{\bot}^* > \max_{y} g_{y}^*$ and $\bI_{h=y}=1$ since the data is realizable and when we defer the human must be correct. Examining the first term and taking the limit:

\begin{align*}
    & \lim_{u \to \infty} \bE[ L_{RS}(um^*,ur^*,x,y,h)|r^*=1] \bP(r^*=1) \\
    &= \lim_{u \to \infty} \bE[  -2 \log\left(\frac{\exp(ug_{y}^*(x)) + \bI_{h = y} \exp(ug_{\bot}^*(x)) }{\sum_{y' \in \mathcal{Y} \cup  \bot}\exp(ug_{y'}^*(x))} \right) |r^*=1] \bP(r^*=1) \\
    &= \lim_{u \to \infty} \bE[  -2 \log\left(\frac{\exp(ug_{y}^*(x)) +  \exp(ug_{\bot}^*(x)) }{\sum_{y' \in \mathcal{Y} \cup  \bot}\exp(ug_{y'}^*(x))} \right) |r^*=1] \bP(r^*=1) \\
    &= \bE[  -2 \log\left(1 \right) |r^*=1] \bP(r^*=1) = 0 \quad  \textrm{(applying monotone convergence theorem)}
\end{align*}

The second term is when $r^*=0$, then we must have $g_{y}^* > \max_{y'\in(\mathcal{Y}\backslash y)\cup\bot} g_{y'}^*$  since the data is realizable. Examining the second term and taking the limit:

\begin{align*}
    & \lim_{u \to \infty} \bE[ L_{RS}(um^*,ur^*,x,y,h)|r^*=0] \bP(r^*=0) \\
    &= \lim_{u \to \infty} \bE[  -2 \log\left(\frac{\exp(ug_{y}^*(x)) + \bI_{h = y} \exp(ug_{\bot}^*(x)) }{\sum_{y' \in \mathcal{Y} \cup  \bot}\exp(ug_{y'}^*(x))} \right) |r^*=0] \bP(r^*=0) \\
    &= \bE[  -2 \log\left(1 \right) |r^*=0] \bP(r^*=0) = 0 \quad  \textrm{(applying monotone convergence theorem)}
\end{align*}
Thus combining the above two derivations, we obtain:
\begin{align}
\nonumber
&L_{\mathrm{def}}^{0{-}1}(\hat{m},\hat{r}) \leq 0. 
\end{align} 
We just proved that the optimal solution from minimizing \realizablesurrogate leads to a zero error solution in terms of system error which proves that the loss is realizable ($\cM,\cR)$-consistent.

\end{proof}

\begin{theorem} \label{apx:the_realiza_not}
    The CrossEntropySurrogate $L_{CE}$ \citep{mozannar2020consistent} is not a realizable $(\mathcal{M},\mathcal{R})$-consistent  surrogate for $L_{\mathrm{def}}^{0{-}1}$.
\end{theorem}

\begin{proof}
    To prove that the surrogate $L_{CE}$ is not realizable-consistent, we will construct an example with a data distribution and a model class closed under scaling such that: 1) there exists a zero error solution in the model class and 2) the minimizer of $L_{CE}$ has non-zero error.

Consider the data distribution illustrated and described in Figure 
\ref{fig:realizable_lce_proof} consisting  of four regions R0,R1,R2 and R3. Each region respectively has mass $1/4+
\alpha, 1/4, 1/4-\alpha,1/4$ . Each region respectively has label $Y=0,Y=1,Y=0,Y=2$. The Human is perfectly accurate on Region 0 and inaccurate on every other region.

\begin{figure}[H]
    \centering
    \includegraphics[scale=1]{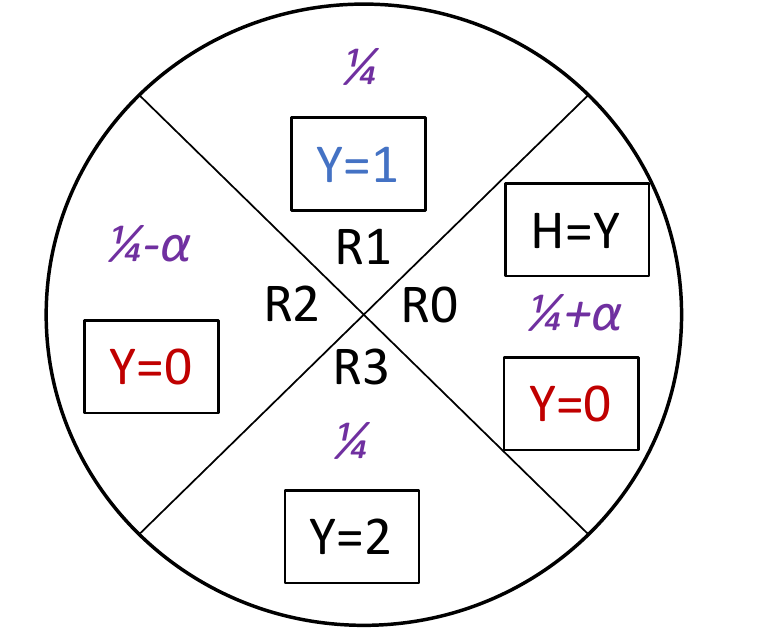}
    \caption{Data Distribution for our example: the data consists of four regions R0,R1,R2 and R3. Each region respectively has mass $1/4+
\alpha, 1/4, 1/4-\alpha,1/4$ . Each region respectively has label $Y=0,Y=1,Y=0,Y=2$. The Human is only accurate on Region 0. }
    \label{fig:realizable_lce_proof}
\end{figure}

We consider a hypothesis class $\mathcal{F}$ parameterized by a scalar $c \in \mathbb{R}$ and four indices each in $i_0,i_1,i_2,i_{\bot} \in \{0,1,2,3\}$. Let $f_i(x) = c \bI\{ x \in R_i\}$, a function $f \in  \mathcal{F}$ defines a rejector and classifier as:  $m(x) = \arg\max \{c 
\cdot f_{i0}(x), c \cdot f_{i1}(x),c \cdot f_{i2}(x)\}$ (ties are decided uniformly randomly) and $r(x) = \bI\{c \cdot f_{\bot}(x) > \max \{c  \cdot f_{i0}(x), c \cdot f_{i1}(x),c \cdot f_{i2}(x)  \}$.  This hypothesis class is closed under scaling. 

The error minimizing function $f^*$ in this hypothesis class is obtained by setting $c>0$, $i_0=2,i_1=1,i_2=3,i_{\bot}=0$ which obtains zero 0-1 error. No solution with $c<0$ is optimal, since the maximum will always coincide with at least two labels and we break ties in a consistent fashion.  
This data distribution and hypothesis class is realizable.

\paragraph{Surrogate solution.} We will argue that one can obtain a lower $L_{CE}$ loss by deviating from the optimal solution $f^*$. 
The intuition for why this is the case is that the $L_{CE}$ penalizes misclassifying points even when they are deferred. 
Hence, when $\alpha$ is sufficiently large, $L_{CE}$ will try to classify the more probable region R0 as label 0 instead of simply deferring on this region and classifying region R2 as label 0.

Consider the function $\hat{f}$ defined with arbitrary $c>0$ and $i_0=0,i_1=1,i_2=3,i_{\bot} =0$---note that this function disagrees with the optimal solution on $i_0$ only. Fixing $c$, we will compute the difference of $L_{CE}$ loss between $\hat{f}$ and $f^*$ with the same $c$, this defines only a deviation in terms of $i_0$. We will compute the difference in each region separately.

\textbf{Region 1 and Region 3:} On both region 1 and region 3, the difference will be shown to be zero. In both regions, the human is incorrect and note that $i_1$ and $i_2$ are identical in both solutions. The loss of $\hat{f}$ in region 1 is:
\begin{equation*}
- \frac{1}{4} \log \left( \frac{e^c}{3+e^c} \right)
\end{equation*}
this is the same as the loss of $f^*$, by symmetry the loss is the same in region 3.

We will now compute the sum of the difference in region 2 and region 0:

\textbf{Region 2:}   In this region the human is also incorrect, the difference in the loss of $\hat{f}$ and $f^*$ is:

\begin{align*}
\bE_{x \in R2}[ L_{CE}(f^*) -L_{CE}(\hat{f}) ]= (\frac{1}{4} - \alpha) \cdot \left(\log \left( \frac{1}{4} \right)  - \log \left( \frac{e^c}{3+e^c} \right)   \right) \in [-(\frac{1}{4} - \alpha) \log(4), 0]
\end{align*}

\textbf{Region 0:} In this region the human is correct, the difference is :
\begin{align*}
 &\bE_{x \in R0} [L_{CE}(f^*) -L_{CE}(\hat{f})] \\&  = (\frac{1}{4} + \alpha) \cdot \left(- \log \left( \frac{1}{3+e^c} \right)  - \log \left( \frac{e^c}{3+e^c} \right) + \log \left( \frac{e^c}{2+2e^c} \right) + \log \left( \frac{e^c}{2+2e^c} \right)    \right) 
\end{align*}

To compute the difference in the loss between $\hat{f}$ and $f^*$, we sum the difference in Region 2 and Region 0:

\begin{align*}
    &L_{CE}(f^*) -L_{CE}(\hat{f})  \\
    &= \frac{1}{4} \left( \log \left( \frac{1}{4} \right)  - \log \left( \frac{e^c}{3+e^c} \right)  - \log \left( \frac{1}{3+e^c} \right)  - \log \left( \frac{e^c}{3+e^c} \right) + 2\log \left( \frac{e^c}{2+2e^c} \right)       \right)  \\
    &+ \alpha \left( - \log \left( \frac{1}{3+e^c} \right)  + 2\log \left( \frac{e^c}{2+2e^c} \right) - \log \left( \frac{1}{4} \right)  \right)  \\
    &= -(\frac{1}{4} + \alpha) \left( \log \left( \frac{1}{3+e^c} \right)  - 2\log \left( \frac{e^c}{2+2e^c} \right)    \right) - \frac{1}{2} \log \left( \frac{e^c}{3+e^c} \right) + (\frac{1}{4} - \alpha) \log \left( \frac{1}{4} \right)
\end{align*}

We can simplify this difference to further become:
\begin{equation*}
   \frac{1}{4}  \left(8 \alpha c - 2 \log(4) -2 (1+4\alpha) \log(1+e^c) + (3+4 \alpha) \log(3+e^c) \right)
\end{equation*}

Note that when $c=0$, the above difference is 0. Let us set $\alpha = 0.125$ for concreteness (other values of $\alpha$ also work, in particular larger values, but not all smaller values). We compute the derivative of the difference with respect to $c$, obtaining:

\begin{align*}
\frac{d}{dc} (L_{CE}(f^*) -L_{CE}(\hat{f})) &= \frac{1}{4} \left( \frac{3.5 e^c}{e^c + 3} - \frac{3 e^c}{e^c + 1} + 1 \right) \\ 
&= \frac{0.375 (2 - e^c + e^{2c}))}{(1 + e^c) (3 + e^c)} > 0
\end{align*}
We just showed that the difference has derivative strictly larger than 0 with respect to $c$, moreover the difference
 is $0$ when $c=0$, thus when $c>0$ the difference is strictly bigger than 0. 

 We just proved that with respect to the surrogate loss $L_{CE}$, the optimal solution with respect to $L_{\mathrm{def}}^{0{-}1}$ is not optimal, thus  the surrogate is not a realizable $(\mathcal{M},\mathcal{R})$-consistent  surrogate for $L_{\mathrm{def}}^{0{-}1}$

\end{proof}

\end{document}